\pdfoutput=1

\documentclass[final]{new_tlp}
\usepackage{mathptmx}
\usepackage[verbose]{placeins}

\usepackage{pgfplots}

\usepackage[utf8]{inputenc}
\usepackage{float} 
\usepackage{amsfonts}
\usepackage{amssymb}
\usepackage{amsmath}
\usepackage{mathtools}
\usepackage{stmaryrd}
\usepackage{csquotes}
\usepackage{url}
\usepackage{algorithm, algpseudocode}

\DeclareMathAlphabet{\mathcal}{OMS}{cmsy}{m}{n}

\usepackage{pifont}

\usepackage{tikz}
\usetikzlibrary{arrows,fit,backgrounds,positioning,patterns,shapes,%
calc,decorations.text,automata, arrows.meta}

\tikzstyle{conceptnode} = [circle, draw, minimum height=0.3cm, minimum width=0.3cm, inner sep=2pt]
\tikzstyle{rolenode} = [circle, draw, thick, minimum height=0.6cm, minimum width=0.6cm, inner sep=2pt]
\tikzstyle{internode} = [ellipse, draw, inner sep=2pt]
\tikzstyle{terminalnode} = [rectangle, draw, minimum height=0.3cm, minimum width=0.3cm, inner sep=2pt]

\tikzstyle{rootnode} = [rectangle,minimum height=0.3cm, minimum width=0.3cm, inner sep=2pt]

\tikzstyle{treenode}=[rectangle split,rectangle split parts=3,draw,text centered]
\tikzstyle{blanknode}=[inner sep=0pt, outer sep=0pt]

\usepackage[obeyDraft,colorinlistoftodos,prependcaption]{todonotes}

\newcommand{\comment}[1]{}

\usepackage{xspace}

\newcommand{\flower}{\ensuremath{\mathcal{F\!L}_{o\hspace{-0.15ex}}\textit{wer}}\xspace}

\newcommand{\EL}{\ensuremath{\mathcal{E\!L}}\xspace}
\newcommand{\ELplus}{\ensuremath{\mathcal{E\!L}^{+}}\xspace}
\newcommand{\ALC}{\ensuremath{\mathcal{ALC}}\xspace}

\newcommand{\wrt}{w.r.t.\@\xspace}

\newcommand{\ie}{i.e.\@\xspace}

\newcommand{\ExpTime}{\textsc{ExpTime}\xspace}
\newcommand{\PTime}{\textsc{PTime}\xspace}
\newcommand{\PSpace}{\textsc{PSpace}\xspace}
\newcommand{\coNP}{\textsc{co-NP}\xspace}

\newcommand{\NC}{\ensuremath{{\mathsf{N}_\mathsf{C}}}\xspace}

\newcommand{\NCCT}{\NC}
\newcommand{\NR}{\ensuremath{{\mathsf{N}_\mathsf{R}}}\xspace}

\newcommand{\NRCT}{\NR}
\newcommand{\I}{\ensuremath{\mathcal{I}}\xspace}
\newcommand{\J}{\ensuremath{\mathcal{J}}\xspace}
\newcommand{\Y}{\ensuremath{\mathcal{Y}}\xspace}
\newcommand{\Z}{\ensuremath{{\mathcal{Z}}}\xspace}
\newcommand{\subsumed}{\ensuremath{\sqsubseteq}\xspace}
\newcommand{\T}{\ensuremath{\mathcal{T}}\xspace}

\newcommand{\subsumedT}{\ensuremath{\subsumed_{\T}}\xspace}

\newcommand{\canonical}[1]{\ensuremath{{\I_{#1,\T}}}\xspace}

\newcommand{\abs}[1]{\ensuremath{{| #1 |}}\xspace}

\newcommand{\FLnull}{\ensuremath{{\mathcal{F\!L}_0}}\xspace}
\newcommand{\flzero}{\FLnull}

\newcommand{\FLbot}{\ensuremath{{\mathcal{F\!L}_\bot}}\xspace}

\newcommand{\SUBSpar}[3]{\ensuremath{{\sc{Subs}(#1, #2, #3)}}\xspace}

\newcommand{\sig}[1]{\ensuremath{\mathsf{sig}(#1)}\xspace}

\newcommand{\domainof}[1]{\ensuremath{{\Delta^{#1}}}\xspace}

\newcommand{\canmodCT}{\ensuremath{\canonical{C}}\xspace}
\newcommand{\canmodDT}{\ensuremath{\canonical{D}}\xspace}
\newcommand{\canmodAT}{\ensuremath{\canonical{A}}\xspace}
\newcommand{\canmodAnullT}{\ensuremath{\canonical{A_0}}\xspace}

\newcommand{\prefixset}[1]{\ensuremath{\mathsf{prefix}(#1)}\xspace}
\newcommand{\pprefixset}[1]{\ensuremath{\mathsf{pprefix}(#1)}\xspace}

\newcommand{\notblocked}[1]{\ensuremath{\mathsf{nb}(#1)}\xspace}

\newcommand{\cmark}{\ensuremath{\text{\ding{51}}}\xspace}%
\newcommand{\xmark}{\ensuremath{\text{\ding{55}}}\xspace}%

\newcommand{\incomp}[1]{\ensuremath{\mathsf{ic}(#1)}\xspace}

\newcommand{\completionstep}[2]{\ensuremath{\mathbin{\vdash^{#2}_{#1}}}\xspace}
\newcommand{\completionstepT}{\ensuremath{ \mathbin{\vdash_{\T}} }\xspace}

\newcommand{\match}[2]{\ensuremath{\mathsf{match}(#1,#2)}}

\newcommand{\depthtreeof}[1]{\ensuremath{\mathsf{depth}(#1)}\xspace}

\newtheorem{definition}{Definition}[section]
\newtheorem{example}{Example}[section]
\newtheorem{lemma}{Lemma}[section]
\newtheorem{theorem}{Theorem}[section]
\newtheorem{proposition}{Proposition}[section]

\newtheorem{remark}{Remark}[section]

\newcommand{\tRule}[1]{\textbf{T#1}}
\newcommand{\bRule}[1]{\textbf{B#1}}

 \title[Efficient TBox Reasoning with Value Restrictions]{Efficient TBox Reasoning with Value Restrictions using the \flower reasoner} 
  \author[Baader et al.\ ]
  {Franz Baader, Patrick Koopmann, Friedrich Michel, Anni-Yasmin Turhan, Benjamin Zarrie{\ss}
  \\
         \email{firstname.lastname@tu-dresden.de}
         }

\jdate{July 2021}
\pubyear{2021}
\pagerange{\pageref{firstpage}--\pageref{lastpage}}
\submitted{May 2020}
\revised{July 2021}

\begin{document}
\label{firstpage}
\maketitle

\begin{abstract}
    The inexpressive Description Logic (DL) \FLnull, which has conjunction and value restriction as its only concept constructors,
had fallen into disrepute when it turned out that reasoning in \FLnull w.r.t.\ general TBoxes is \ExpTime-complete,
i.e., as hard as in the considerably more expressive logic \ALC. In this paper, we rehabilitate \FLnull by presenting
a dedicated subsumption algorithm  for \FLnull, which is much simpler than the tableau-based algorithms employed by highly optimized DL reasoners.
Our experiments show that the performance of our novel algorithm, as prototypically implemented in our \flower reasoner, compares very well with that of
the highly optimized reasoners. \flower can also deal with ontologies written in the extension \FLbot of \FLnull with the
top and the bottom concept by employing a polynomial-time reduction, shown in this paper, which eliminates  top and bottom.
We also investigate the complexity of reasoning in DLs related to the Horn-fragments of \FLnull and \FLbot. 

\smallskip \noindent This paper is under consideration in Theory and Practice of Logic Programming (TPLP).

%
%

\end{abstract}


\section{Introduction}

Description Logics (DLs) \cite{BCNMP03,DLbook} are a well-investigated family of logic-based knowledge representation languages,
which are frequently used to formalize ontologies for application domains such as the Semantic Web \cite{HoPH03}
or biology and medicine \cite{HoSG15}.
To define the important notions of such an application domain as formal concepts, DLs state necessary and
sufficient conditions for an individual to belong to a concept.
These conditions can be Boolean combinations of atomic properties required for the
individual (expressed by concept names) or properties that refer
to relationships with other individuals and their properties
(expressed as role restrictions).
For example, the concept of a parent that has only daughters can be formalized by the concept description
$
C := \exists\textit{child}.\textit{Human} \sqcap \forall\textit{child}.\textit{Female},
$
which uses the concept names \textit{Female} and \textit{Human} and the role name \textit{child} as well as
the concept constructors conjunction ($\sqcap$), existential restriction ($\exists r.D$), and
value restriction ($\forall r.D$).
Constraints on the interpretation of concept and role names can be formulated as general concept
inclusions (GCIs). For example, the GCIs
$
\textit{Human} \sqsubseteq \forall\textit{child}.\textit{Human}
$
and
$
\exists\textit{child}.\textit{Human} \sqsubseteq \textit{Human}
$
say that humans have only human children, and they are the only ones that can have human children.
DL systems provide their users with reasoning services that allow them to derive implicit knowledge from the explicitly
represented one. In our example, the above GCIs imply that elements of our concept $C$ also belong to the
concept $D:= \textit{Human} \sqcap \forall\textit{child}.\textit{Human}$, i.e., $C$ is subsumed by $D$ w.r.t.\
these GCIs.
A specific DL is determined by which kind of concept constructors are available.

In the early days of DL research, the inexpressive DL \flzero, which has only conjunction and value restriction 
as concept constructors, was considered to be the smallest possible DL. In fact, when providing a formal semantics for so-called 
property edges of semantic
networks in the first DL system KL-ONE \cite{BrSc85}, value restrictions were used. For this reason, the language for constructing concepts
in KL-ONE and all of the other early DL systems \cite{BMPAB91,Pelt91,MaDW91,WoSc92} contained \flzero. It came as a surprise when it was 
shown that subsumption reasoning
w.r.t.\ acyclic \flzero TBoxes (a restricted form of GCIs) is \coNP-hard \cite{Nebe90}. The complexity increases when more expressive forms of TBoxes are
used: for cyclic TBoxes to \PSpace \cite{Baad90c,KaNi03} and for general TBoxes consisting of GCIs even to \ExpTime  \cite{BaBL05,Hofm05}.
Thus, w.r.t.\ general TBoxes, subsumption reasoning in \flzero is as hard as subsumption reasoning in \ALC, its closure under negation
\cite{Schi91}.

These negative complexity results for \flzero were one of the reasons why the attention in the research of inexpressive DLs
shifted from \flzero to \EL, which is obtained from \flzero by replacing value restriction with existential restriction as a concept constructor. 
In fact, subsumption reasoning in \EL stays polynomial even in the presence of general TBoxes \cite{Bran04}. 
The reasoning method employed in \cite{Bran04}, which is nowadays called consequence-based reasoning, can be used to establish the
\PTime complexity upper bounds also for reasoning in the extension \ELplus of \EL \cite{BaBL05}. This approach also applies to Horn fragments
of expressive DLs such as $\mathcal{SHIQ}$, for which reasoning is \ExpTime-complete, but consequence-based reasoning approaches
behave considerably better in practice than the usual tableau-based approaches for expressive DLs \cite{Kaza09}. 
The DL \flzero is not Horn,\footnote{%
Actually, reasoning in its Horn fragment is \PTime \cite{DBLP:conf/aaai/KrotzschRH07,HORN-DLs}.
}
but it shares with \ELplus and Horn-$\mathcal{SHIQ}$ that (general) TBoxes have canonical models, i.e., models such
that a subsumption relationship between concept names follows from the TBox if and only if it holds in the canonical model.
Consequence-based reasoning basically generates these models. However, whereas the canonical models for \EL and Horn-$\mathcal{SHIQ}$
are respectively of polynomial and exponential size, the canonical models for \flzero, called least functional models \cite{DBLP:conf/gcai/BaaderGP18}, 
may be infinite.

In this paper we build on and extend the results from \cite{MTZ-RuleML-19}. We devise a novel algorithm for deciding subsumption \wrt general \FLnull TBoxes, 
describe a first implementation of it in the new \flower reasoner,\footnote{\url{ https://github.com/attalos/fl0wer}} 
and report on an evaluation of \flower on a large collection of ontologies, which shows that \flower competes well with existing highly optimized DL reasoners. 
Basically, our new algorithm generates \enquote{large enough} parts of the least functional model and achieves termination
using a blocking mechanism similar to the ones employed by tableau-based reasoners.
The key idea of the implementation is to apply the TBox statements like rules and to use a variant of the well-known Rete algorithm for rule application \cite{DBLP:journals/ai/Forgy82}, 
adapted to the case without negation. 
To create a large set of challenging \FLnull ontologies we have used, on the one hand, the OWL 2 EL ontologies of the OWL 
reasoner competition \cite{parsia2017owl} transformed into \FLnull by exchanging the quantifier and omitting too small ontologies as too easy.
On the other hand, we have extracted \FLnull sub-ontologies of decent size from the ontologies of the 
Manchester OWL Corpus (MOWLCorp).\footnote{%
\url{ https://zenodo.org/record/16708}}

In the next section, we introduce \FLnull and its extension \FLbot with the top ($\top$) and the bottom~($\bot$) concepts.
We recall the characterization of subsumption based on least functional models from \cite{DBLP:conf/gcai/BaaderGP18},
introduce a normal form for \FLnull TBoxes,
and show that the bottom concept $\bot$ and the top concept $\top$ can be simulated by such TBoxes.
In Section~\ref{sec:subs}, we introduce our new algorithm, and prove that it is sound, complete, and terminating.
Section~\ref{sect:horn} considers the Horn fragments of \FLnull and \FLbot. First, we show that, for Horn-\FLnull, our
algorithm can be restricted such that it runs in polynomial time. A polynomial upper bound for subsumption in Horn-\FLnull
has already been shown in \cite{DBLP:conf/aaai/KrotzschRH07,HORN-DLs} for an extension of Horn-\FLnull that contains $\bot$.
However, this extension is weaker than Horn-\FLbot. In fact, we also show in Section~\ref{sect:horn} that subsumption in
Horn-\FLbot is \PSpace-complete, and that it becomes \ExpTime-complete in a small extension of Horn-\FLbot. 
Section~\ref{sec:rete} describes how to realize our novel algorithm based on Rete, and
Section~\ref{sec:eval} presents our experimental results, which evaluate several optimizations of the algorithm, and
compare its performance with that of existing highly optimized DL reasoners.

\section{Preliminaries on \FLnull and Extensions}\label{sec:lfm}

\newcommand{\FLnullBot}{\FLbot}
\newcommand{\NRstar}{\ensuremath{\NR^*}}

We introduce the DL \FLnull, 
recall the characterization of subsumption based on least functional models from \cite{DBLP:conf/gcai/BaaderGP18},
introduce a normal form for \FLnull TBoxes,
and show that the bottom concept $\bot$ and the top concept $\top$ can be simulated by such TBoxes.

\subsection{Syntax, Semantics, and Functional Interpretations}

\newcommand{\sigC}[1]{\textsf{sig}_\mathsf{C}(#1)}
\newcommand{\sigR}[1]{\textsf{sig}_\mathsf{R}(#1)}

\paragraph*{Syntax.}
Let \NC and \NR be disjoint, at most countably infinite sets of \emph{concept names} and \emph{role names}, respectively. 
An \emph{\FLnull concept description} (\emph{concept} for short) $C$ is built according to the following syntax rule
\begin{align*}
	C ::= A \mid 
              C \sqcap C \mid \forall r. C, \text{~ where } A \in \NC, r \in \NR.
\end{align*}
Additionally allowing the use of the top concept $\top$ and the bottom concept $\bot$
in the above rule yields the DL 
\FLbot.
A \emph{general concept inclusion} (GCI) for any of these DLs is of the form $C \subsumed D$, 
where $C$ and $D$ are concepts of the respective DL. A \emph{TBox} is a finite set of GCIs. 
The \emph{signature} $\sig{C}$ ($\sig{\T}$) of a concept $C$ (TBox $\T$) is the set of concept and role names occurring in $C$ ($\T$). For convenience, we use further functions to refer only to the concept names and only to the role names in an expression. For a concept or TBox $E$, we set $\sigC{E}=\sig{E}\cap\NC$ and $\sigR{E}=\sig{E}\cap\NR$.

The expression $\forall r. C$ is called a \emph{value restriction}. 
For nested value restrictions we use the following notation: given a word $\sigma = r_1\cdots r_m \in \NRstar$, $m \geq 0$, over the alphabet $\NR$ of role names, and a concept $C$, we write $\forall \sigma. C$ as an abbreviation of  $\forall r_1. \cdots \forall r_m. C$. 
For the empty word $\epsilon$, we have $\forall\epsilon.C=C$.

\paragraph*{Semantics.}
An interpretation $\I$ is a pair $\I = (\domainof{\I},\cdot^\I)$, consisting of a non-empty set 
\domainof{\I} (the \emph{domain} of \I) and an \emph{interpretation function} $\cdot^\I$ that 
maps every concept name $A \in \NC$ to a subset $A^\I \subseteq \domainof{\I}$ of the domain,  
and every role name $r \in \NR$ to a binary relation $r^\I \subseteq \domainof{\I} \times \domainof{\I}$.
The interpretation function is extended to (complex) concepts as follows:
\begin{align*}
	(C \sqcap D)^\I & := C^\I \sqcap D^\I, 
	\qquad 
	\top^\I := \domainof{\I}, \qquad
        \bot^\I := \emptyset,\text{ and} 
	\\
	(\forall r. C)^\I & := \{ d \in \domainof{\I} \mid 
	\forall e \in \domainof{\I}. (d,e) \in r^\I \ \rightarrow \ e \in C^\I  \}.
\end{align*}
The GCI $C \subsumed D$ is \emph{satisfied} in $\I$, denoted as $\I \models C \subsumed D$, if
$C^\I \subseteq D^\I$. The interpretation \I is a \emph{model} of the TBox~\T, denoted as $\I \models \T$, 
if $\I$ satisfies all GCIs in \T. 
The concept $C$ is \emph{subsumed} by the concept $D$ \emph{\wrt \T}, 
denoted as $C \subsumedT D$, if $C^\I \subseteq D^\I$ is satisfied in all models \I of \T.

To decide subsumption in \FLnull, it is sufficient to consider so-called functional interpretations, which are tree-shaped interpretations in which every element has exactly one child for each role name. In such interpretations, domain elements are identified by sequences of role names.

\begin{definition}
	An interpretation $\I = (\domainof{\I},\cdot^\I)$ is called a \emph{functional interpretation} if $\domainof{\I} = \NRstar$ and for all $r \in \NR$, $r^\I = \{ (\sigma,\sigma r) \mid \sigma \in \NRstar \}$. It is called a \emph{functional model} of the \FLnull concept $C$ \wrt the \FLnull TBox \T 
	if $\I \models \T$ and $\epsilon \in C^\I$. 
	For two functional interpretations $\I$ and $\J$ we write 
	\[
	\I \subseteq \J \text{ ~if~ } A^\I \subseteq A^\J \text{ for all } A \in \NC.
	\]
\end{definition}
  
The notion of a functional interpretation fixes the domain and the interpretation of role names. 
Thus, a functional interpretation is uniquely determined by the interpretation of the concept names.
Given a family $(\I_i)_{i\geq 0}$ of functional interpretations, their \emph{intersection} $\J := \bigcap_{i\geq 0} \I_i$
is the functional interpretations that satisfies $A^\J = \bigcap_{i\geq 0} A^{\I_i}$. 

\begin{lemma}[see \cite{DBLP:conf/gcai/BaaderGP18}]
Given an \FLnull concept $C$ and
an \FLnull TBox \T, the functional models of $C$ \wrt \T are closed under intersection. In particular, this
implies that there exists a \emph{least functional model $\canmodCT$} of $C$ \wrt \T, i.e., a functional model
of $C$ \wrt \T such that $\canmodCT \subseteq \J$ holds for all functional models $\J$ of $C$ \wrt \T.
\end{lemma}

In \cite{DBLP:conf/gcai/BaaderGP18}, subsumption in \FLnull was characterized as inclusion of least functional models as follows:
given \FLnull concepts $C, D$ and an \FLnull TBox \T, we have 
\begin{equation}\label{sub:char:eqn}
C \subsumedT D\ \ \text{iff}\ \ \canmodDT \subseteq \canmodCT.
\end{equation}
For our purposes, the following characterization of subsumption turns out to be more useful.

\begin{theorem}\label{sub:char:least-functional}
Given \FLnull concepts $C, D$ and an \FLnull TBox \T, we have $C \subsumedT D$ 
iff $\varepsilon\in D^{\canmodCT}$.
\end{theorem}

\begin{proof}
Assume that $C \subsumedT D$. Then  $\varepsilon\in C^{\canmodCT}$ (which we know since $\canmodCT$ is a
functional model of $C$ \wrt \T) implies $\varepsilon\in D^{\canmodCT}$ since $\canmodCT$ is a model of \T.
Conversely, $\varepsilon\in D^{\canmodCT}$ implies that $\canmodCT$ is a functional model of $D$ \wrt \T,
and thus $\canmodDT\subseteq \canmodCT$, which yields $C \subsumedT D$ by \eqref{sub:char:eqn}. \hfill
\end{proof}

\subsection{Normal Forms for \FLbot and \FLnull Concepts and TBoxes}
\label{ssec:normal-forms}

An \FLbot \emph{concept} is in \emph{normal form} if it is of the form 
\begin{itemize}
\item 
   $\top$ or $\bot$, or
\item
  a non-empty conjunction of concepts of the form $\forall r.\bot$, $A$, $\forall r.A$, where $A\in\NC$ and $r\in\NR$.
\end{itemize}
An \FLbot \emph{TBox} is in \emph{normal form} if it contains only GCIs of the form $C\sqsubseteq D$, where $C, D$ are
in normal form, and $C$ is not $\bot$ and $D$ is not $\top$.
In addition,
\FLnull concepts (TBoxes) in normal form are \FLbot concepts (TBoxes) in normal form that contain neither $\top$ nor $\bot$.

It is easy to see that every (\FLbot or \FLnull) TBox $\T$ can be transformed in
linear time into a TBox in normal form such that all subsumption relationships in the signature of
$\T$ are preserved. For this, one removes tautological GCIs with $\bot$ on the left-hand side
or $\top$ on the right-hand side, and flattens value-restrictions $\forall r.E$ with $E\not\in\NC\cup\{\bot\}$.
To \emph{flatten} an occurrence of $\forall r.E$ in a GCI $C\sqsubseteq D$ means that $E$ is replaced by a fresh concept name $A_E$. 
If the occurrence is within $C$, then the GCI $E\sqsubseteq A_E$ is added to the TBox, and otherwise $A_E\sqsubseteq E$.

It is well-known that subsumption between complex concepts can be reduced in linear time to subsumption between concept names.
In fact, we have $C \subsumedT D$ iff $A\sqsubseteq_{\T'} B$, where $A, B$ are concept names not occurring in $C$, $D$, or $\T$,
and $\T'$ is obtained from $\T$ by adding the GCIs $A\sqsubseteq C$ and $D\sqsubseteq B$.

\begin{proposition}\label{prop:normal-form}
Subsumption in \FLbot (\FLnull) w.r.t.\ TBoxes
can be reduced in linear time 
to subsumption of concept names
w.r.t.\ \FLbot (\FLnull) TBoxes in normal form.
\end{proposition}

For subsumption between concept names $A, B$ in the DL \FLnull, the characterization of subsumption given in 
Theorem~\ref{sub:char:least-functional} means that, to decide whether $A\subsumedT B$ holds, it is sufficient to check
whether the root of $\canmodAT$ is contained in $B^{\canmodAT}$, i.e., whether the label of this root
contains the concept name $B$.

\subsection{Reducing Subsumption in \FLnullBot to Subsumption in \FLnull}\label{ssec:reductions}

Subsumption between concept names in \FLnullBot can be reduced to subsumption in \FLnull using the 
following transformation rules on \emph{normalized} \FLnullBot TBoxes $\T$:
\begin{enumerate} 
 \item[\tRule{1}] 
       Replace $\bot$ and $\top$ everywhere by the fresh concept names $A_\bot$ and $A_\top$, respectively;
 \item[\tRule{2}]\label{red:item:two} 
       add the axioms $A_\bot\sqsubseteq B$ for all $B\in\sigC{\T}$; 
 \item[\tRule{3}]\label{red:item:three} 
       add the axioms $B\sqsubseteq A_\top$ and $A_\top\sqsubseteq\forall r.A_\top$ for all
       $B\in\sigC{\T}$ and all $r\in\sigR{\T}$.
\end{enumerate}
We denote the TBox resulting from this transformation as $\FLnull(\T)$.

\begin{lemma}\label{lem:fl0bot}
 For all \FLnullBot TBoxes $\T$ in normal form and all concept names $A$, $B$ occurring in $\T$, we
 have $A\sqsubseteq_\T B$ iff $A\sqsubseteq_{\FLnull(\T)} B$.
\end{lemma}
\begin{proof}
\enquote{$\Leftarrow$}: 
Assume that $A\not\sqsubseteq_\T B$. Then there is a model $\I$ of $\T$ such that $A^\I\not\subseteq B^\I$.
We modify $\I$ to an interpretation $\J$ by setting ${A_\bot}^{\J} := \emptyset$ and
${A_\top}^{\J} := \Delta^\I$, and leave the domain as well as the interpretation of the other concept names and the role 
names as in $\I$. It is easy to see that $\J$ is a model of $\FLnull(\T)$ that satisfies
$A^{\J} = A^\I\not\subseteq B^\I = B^{\J}$.

\enquote{$\Rightarrow$}:
Assume that $A\not\sqsubseteq_{\FLnull(\T)} B$, and let $\I$ be a model of $\FLnull(\T)$ that contains an element $d_0$ with
$d_0 \in A^\I\setminus B^\I$. We may assume without loss of generality that all elements of $\Delta^\I$ are reachable from $d_0$ 
via a path of roles in $\sigR{\T}$.
Due to the GCIs introduced by \tRule{3}, 
$d_0\in A^\I$ yields  $d_0\in {A_\top}^{\I}$, 
and  thus $d\in {A_\top}^{\I}$ holds for all $d\in \Delta^\I$. We also know that $d_0\not\in{A_\bot}^{\I}$ since otherwise
the GCI $A_\bot\sqsubseteq B$ added by \tRule{2} would yield $d_0\in B^\I$, contradicting our assumption that $d_0$
is a counterexample to the subsumption. The interpretation
$\J$ is obtained from $\I$ by removing all elements of ${A_\bot}^{\I}$.
Then $d_0$ is an element of $\Delta^{\J}$ and it satisfies $d_0 \in A^{\J}\setminus B^{\J}$. Thus, it remains to show
that $\J$ is a model of $\T$.

First, note that ${A_\top}^{\J} = \Delta^{\J} = \top^{\J}$ and ${A_\bot}^{\J} = \emptyset = \bot^{\J}$. This implies that it is enough to prove that the GCIs from $\T$ transformed by 
\tRule{1}, which are satisfied by $\I$ since it is a model of $\FLnull(\T)$, 
are also satisfied by $\J$.
For this, it is in turn sufficient to show that, for all concepts $C$ in normal form occurring in $\FLnull(\T)$ and all $d\in \Delta^{\J}$
we have $d\in C^\I$ iff $d\in C^{\J}$. For concept names this is trivial by the definition of $\J$. 
Thus, consider a value restriction of the form $\forall r.A_1$.

First, assume that $d\in (\forall r.A_1)^\I$, but $d\not\in (\forall r.A_1)^{\J}$. 
Then there is an element $e\in \Delta^{\J}$ with $(d,e)\in r^{\J}$, but $e\not\in {A_1}^{\J}$. However, since $e\in \Delta^{\J}$, we already
know that $e\not\in {A_1}^{\J}$ implies $e\not\in {A_1}^{\I}$. Since we also have $(d,e)\in r^{\I}$, this contradicts our assumption that
$d\in (\forall r.A_1)^\I$.

Second, assume that $d\in (\forall r.A_1)^{\J}$, but $d\not\in (\forall r.A_1)^\I$.
Then there is an element $e\in \Delta^{\I}$ with $(d,e)\in r^{\I}$, but $e\not\in {A_1}^{\I}$.
If $e\in \Delta^{\J}$, then we also have $(d,e)\in r^{\J}$ and $e\not\in {A_1}^{\J}$, which contradicts our assumption that
$d\in (\forall r.A_1)^{\J}$. Otherwise, we must have $e\in {A_\bot}^{\I}$ since $e$ was removed. But then the GCIs introduced by \tRule{2}
yield $e\in {A_1}^{\I}$, contradicting our assumption on $e$.\footnote{%
Note that $A_1$ cannot be $A_\top$ since a value restriction of the form $\forall r.\top$ is not normalized.
}
 \hfill \end{proof}

Since normalization of an \FLbot TBox and the transformation into an \FLnull TBox described in this subsection
are polynomial, we obtain the following result.

\begin{theorem}
Subsumption in \FLbot can be reduced in polynomial time to subsumption in \FLnull.
\end{theorem}

\section{Subsumption Algorithm for \FLnull with General TBoxes}\label{sec:subs}

\newcommand{\scap}{\sqcap}
\newcommand{\Imc}{\I}
\newcommand{\squbseteq}{\sqsubseteq}
\newcommand{\sqscap}{\sqcap}
\newcommand{\sqsubsteeq}{\sqsubseteq}
\newcommand{\sqsusbeteq}{\sqsubseteq}
\newcommand{\Ymc}{\Y}
We define a decision procedure for subsumption of two concepts \wrt a TBox based 
on a finite representation of the least functional model obtained by 
\enquote{applying} GCIs like rules. By Proposition~\ref{prop:normal-form}, it is sufficient to focus on \FLnull TBoxes in normal form and subsumption between concept names. We can then use Lemma~\ref{lem:fl0bot} to extend the applicability of our algorithm to \FLnullBot.

In the remainder of this section, $\T$ denotes a \FLnull TBox in normal form, and we focus on the task of deciding $A\sqsubseteq_\T B$ for two concept names $A$, $B$ occurring in $\T$.
For the sake of simplicity, we assume in this section that \NC and \NR consist exactly of the concept and role names occurring in $\T$.
In particular, this means that \NC and \NR are finite and their cardinalities are bounded by the size of~$\T$.

The algorithm computes a finite subtree of the tree $\canmodAT$ 
such that one can read off the named subsumers (concept names) of $A$ \wrt \T at 
the root.  
The finite structure that the algorithm operates on is called \emph{partial functional interpretation}. This is similar to a functional interpretation, except that the domain is a \emph{finite} prefix-closed subset of $\NRstar$, that is, a finite tree.
\begin{definition}
	{An} interpretation $\Y = (\domainof{\Y},\cdot^\Y)$ is a \emph{partial functional interpretation} iff 
	$\domainof{\Y} \subseteq \NRstar$ is a \emph{finite} prefix-closed set and 
	$r^\Y = \{ (\sigma,\sigma r) \mid \sigma r \in \domainof{\Y} \}$ for all $r \in \NR$.
\end{definition}
Note that, as with functional interpretations, the interpretation of the role names is already determined by the domain. Thus, it suffices to give the domain and the interpretation of concept names to fix a partial interpretation.

Informally, the algorithm for deciding $A \subsumedT B$ proceeds as follows:
it starts with a partial functional interpretation $\Y$ that has $\epsilon$ as only domain element, and for which $A^\Y=\{\epsilon\}$. 
In each iteration, a domain element $d$ of the current tree $\Y$ and a single GCI $C\sqsubseteq D$ from $\T$ is chosen such that $d$ matches $C$ and does not match $D$. The tree is then extended so that $d$ matches $D$. The extension can affect both the domain and the interpretation of concept names. The method proceeds in such a way that, for every generated tree $\Y$, the invariant  $\Y \subseteq \canmodAT$ is satisfied. Termination is established by blocking further extensions for duplicate elements. The algorithm terminates if the following holds for every non-blocked element $d$ and every GCI $C\sqsubseteq D$ in $\T$: if $d$ matches $C$, then $d$ also matches $D$. Soundness and completeness is shown by establishing a correspondence between the nodes in the final tree and nodes in the least function model of $A$ \wrt \T.
To describe the procedure more formally, we must define the following notions:
\begin{enumerate}
	\item the condition under which a domain element of a partial interpretation matches a concept,
	\item the extension of the tree to achieve a match of an element with the right-hand side of a GCI, and
	\item the conditions that distinguish blocked from non-blocked elements.
\end{enumerate}
To address the first point, we introduce the following auxiliary notions.
\begin{definition}\label{def:match}
	Let $\Y = (\domainof{\Y},\cdot^\Y)$ be a partial functional interpretation and $D$ a concept in normal form. 
	The set of elements in $\domainof{\Y}$ that \emph{match} $D$, denoted by $\match{D}{\Y}$, is defined inductively as follows:
	\begin{align*}
	\match{A}{\Y} :=~&  A^\Y \text{ for all } A\in \NCCT; 
	\\
	\match{\forall r. A}{\Y} :=~& \{ \sigma \in \domainof{\Y} \mid 
	\sigma r \in A^\Y  \} 
	%
	\text{ for all } r \in \NRCT \text{ and } A \in \NCCT;
	\\
	\match{C_1 \sqcap C_2}{\Y} :=~& \match{C_1}{\Y} \cap \match{C_2}{\Y}. 
	\smallskip
	\end{align*}
\end{definition}
Since $\Y$ is partial functional (\ie has \emph{at most} one child per node for each role name), it is easy to see that $\sigma \in \match{C}{\Y}$ 
implies $\sigma \in C^\Y$. The converse need not be true, as $\sigma$ may have no $r$-child in $\domainof{\Y}$.
We say that $\sigma \in \domainof{\Y}$ \emph{violates} the GCI $C \subsumed D$ iff $\sigma \in \match{C}{\Y}$ 
and $\sigma \notin \match{D}{\Y}$. In this case, $\sigma$ is called an \emph{incomplete element}.
Given a TBox \T in normal form and a partial functional interpretation $\Y$, we define the \emph{set of all incomplete elements} as follows:
\begin{align*}
\incomp{\Y,\T} := \{ \sigma \in \domainof{\Y} \mid \text{ there is } C \subsumed D \in \T \text{ such that } \sigma \text{ violates } C \subsumed D  \}.
\end{align*}
Intuitively, the elements in $\incomp{\Y,\T}$ are those eligible for an extension of $\Y$ towards building a representation of the least functional model, while those in $\domainof{\Y} \setminus \incomp{\Y,\T}$ are not. 
As an additional filter for extensions, we define a blocking condition. First, we introduce auxiliary notions for the blocking mechanism consisting of the standard notions of prefix, proper prefix, and a strict total order on $(\NRCT)^*$. 

Let $\sigma,\rho \in \NRstar$. The \emph{length of an element} $\sigma \in \NRstar$ is denoted by $\abs{\sigma}$. We write $\rho \in \prefixset{\sigma}$ if $\sigma = \rho \widehat{\sigma}$ 
for some $\widehat{\sigma} \in \NRstar$, 
and $\rho \in \pprefixset{\sigma}$ if $\rho \in \prefixset{\sigma}$ and $\rho \neq \sigma$. In the latter case, $\rho$ is called a 
\emph{proper prefix} of $\sigma$. 
Let $\prec$ be any total order on $\NRstar$ such that $\abs{\sigma}<\abs{\rho}$ implies $\sigma\prec\rho$ for all $\sigma, \rho\in\NRstar$. 
Since $\NR$ is finite, this implies that, for any element of $\sigma\in \NRstar$, there are only finitely many elements $\rho$ such that
$\rho \prec\sigma$. In particular, the order  $\prec$ is well-founded.
 
For a (partial) functional interpretation $\Y=(\domainof{\Y},\cdot^\Y)$ and $\sigma\in\domainof{\Y}$, we define the \emph{label} of $\sigma$ in $\Y$ as 
 $\Y(\sigma) := \{ A \in \NCCT \mid \sigma \in A^\Y \}$. The cardinality of $\Y(\sigma)$ is bounded by the size of $\T$, and thus 
there can be only exponentially many different such labels.
\begin{definition}\label{def:block}
	Let $\Y = (\domainof{\Y},\cdot^\Y)$ be a partial functional interpretation. 
	The \emph{set of all blocked elements} in $\domainof{\Y}$ is defined 
        by induction over the well-founded order $\prec$:
	\begin{enumerate}
		\item[\bRule{1}] 
		The least element $\epsilon$ is not blocked. 
		\item[\bRule{2}] 
                        The element $\sigma\in\domainof{\Y}$ is blocked if 
			there exists $\omega \in \domainof{\Y}$ with $\omega \prec \sigma$ such that 
			$\Y(\sigma) = \Y(\omega)$ and $\omega$ is not blocked. 
			
		\item[\bRule{3}] 
			Furthermore, the element  $\sigma\in\domainof{\Y}$ is blocked if there exists $\rho \in \pprefixset{\sigma}$ such that $\rho$ is blocked. 
	\end{enumerate}
Only elements of $\domainof{\Y}$ for which \bRule{1} or \bRule{2} holds can be blocked. All other elements
are \emph{non-blocked} elements, which are collected in the set $\notblocked{\Y}$.
\end{definition}	

Condition~\bRule{2} corresponds to \emph{anywhere blocking} in classical tableau algorithms: intuitively, if there are two nodes with the same label, it suffices to reason only on one of them, and the ordering decides which one is used. Condition~\bRule{3} corresponds to \emph{ancestor blocking}: if it is already decided that a node can be ignored, it is not necessary to consider its descendants either. 
Nodes blocked due Condition~\bRule{2} 
are called \emph{directly blocked}, while nodes blocked due Condition~\bRule{3} 
are called \emph{indirectly blocked}.

Next, we define what an extension step is. Such a step expands a single non-blocked and incomplete element in a partial functional interpretation. 
\begin{definition}
	Let $\Y$  
        be a partial functional interpretation, $\T$ a TBox in normal form,  $m,n \geq 0$ and
	\begin{align*}
       \alpha\ \mbox{a GCI in $\T$ of the form}\ \
	\alpha = C \subsumed 
	\left(A_1 \sqcap \cdots \sqcap A_m \sqcap \forall r_1. B_1 \sqcap \cdots \forall r_n. B_n \right).  
	\end{align*}
	In addition, let $\sigma \in \notblocked{\Y} \cap \incomp{\Y,\T}$ 
 	be a non-blocked, incomplete element in \Y violating $\alpha$. 
	Then, the \emph{expansion of $\alpha$ at $\sigma$ in $\Y$} is the partial 
interpretation $\Z$ defined by
	\begin{itemize}
		\item $\domainof{\Z} = \domainof{\Y} \cup \{ \sigma r_1,\ldots,\sigma r_n  \}$;
		\item $A_i^\Z = A_i^\Y \cup \{ \sigma \}$ for all $i = 1,\ldots,m$;
		\item $B_i^\Z = B_i^\Y \cup \{ \sigma r_j \mid 1\leq j\leq n, B_j=B_i\}$ for all $i = 1,\ldots,n$; and
		\item $Q^\Z = Q^\Y$ for all $Q \in \NCCT \setminus \{ A_1,\ldots A_m,B_1,\ldots,B_n \}$.
	\end{itemize}
	A partial functional interpretation $\Z$ 
	is a \emph{$\T$-completion} of $\Y$, written as $\Y \completionstepT \Z$, 
	iff $\Z$ is an expansion of some $\alpha\in\T$ at some $\sigma'\in\notblocked{\Y}\cap\incomp{\Y,\T}$. We denote by $\completionstepT^*$ the reflexive transitive closure of $\completionstepT$ and call $\Z$ with $\Y \completionstepT^* \Z$ \emph{complete} 
if every incomplete element is blocked, i.e., $\notblocked{\Y_n} \cap \incomp{\Y_n,\T} = \emptyset$. 
\end{definition}
Depending on the choice of $\sigma$ and the GCI, there can be several \T-completions of~$\Y$. Also note that it is guaranteed that either 
$\notblocked{\Y} \cap \incomp{\Y,\T} = \emptyset$ or there exists a \T-completion of~$\Y$. Thus, in case a given $\Z$ with $\Y \completionstepT^* \Z$ 
is not complete, it can be further completed.

Given the input $A_0,B_0 \in \NC$ and $\T$, the algorithm \SUBSpar{A_0}{B_0}{\T} for deciding $A_0 \subsumedT B_0$ computes a sequence of $\T$-completions until it reaches a complete partial functional interpretation, i.e., one where no non-blocked element violates any GCI from $\T$.
The algorithm starts with the following partial functional interpretation:
\begin{align}\label{eq:inittree}
\domainof{\Y_0} := \{ \epsilon \}; \quad 
A_0^{\Y_0} := \{ \epsilon \} 
\quad  \text{ and } \quad 
B^{\Y_0} := \emptyset \text{ for all } B \in \NCCT \setminus \{ A_0 \},
\end{align}	
and computes a sequence 
\begin{align*}
	\Y_0 \completionstepT \Y_1 \completionstepT \cdots \Y_{(n-1)}\completionstepT \Y_n
\end{align*}
such that $\Y_n$ is complete 
in the sense introduced above.
It answers \enquote{yes} if $B_0 \in \Y_n(\epsilon)$ (or equivalently $\epsilon \in B_0^{\Y_n}$) and \enquote{no} otherwise. 

\begin{example}
	In this example, we illustrate the completion steps and how the blocking conditions are applied.
	Let $\NC = \{ A,B,K,L,M \}$ and $\NR = \{ r,s \}$.
	The TBox \T is defined as follows:
	\begin{align*}
	\T := \{ ~~~~~~~~
	 A &\subsumed  \forall r. A,  &A \subsumed~&  B, \\ 
	 A &\subsumed  \forall s. K,  & K \subsumed~&  \forall s. A, \\
	 \forall s. B &\subsumed  L,  &\forall s. L   \subsumed~& M ~~~~~~~~ \}.
	\end{align*}
	One can verify that 
	$$
	A \subsumedT M.
	$$
	In fact, the GCIs $A \subsumed \forall s. K$, $K \subsumed \forall s. A$ and $A \subsumed B$ yield $A \subsumedT \forall s. \forall s. B$. 
	Using $\forall s. B \subsumed L$ and $\forall s. L \subsumed M$,  we then obtain $A \subsumedT M$. 
	\begin{figure}
		\framebox[0.95\textwidth]{
		
				\centering
			
				\begin{tikzpicture}
				\node[treenode] (eps0) {$\boldsymbol{\epsilon}$ \nodepart[text width=1cm]{second} $\{ A  \} $ \nodepart{third} \phantom{\cmark}  };
				
				\node[treenode, below=0.5cm of eps0] (eps1) {$\boldsymbol{\epsilon}$ \nodepart[text width=1cm]{second} $\{ A  \} $ \nodepart{third} \phantom{\cmark}  };
				\node[treenode,right=0.2cm of eps1] (r1) {$\boldsymbol{r}$ \nodepart[text width=1cm]{second} $\{ A  \} $ \nodepart{third} \xmark  };
				
				\node[treenode, below=0.5cm of eps1] (eps2) {$\boldsymbol{\epsilon}$ \nodepart[text width=1cm]{second} $\{ A,B  \} $ \nodepart{third} \phantom{\cmark}  };
				\node[treenode,right=0.2cm of eps2] (r2) {$\boldsymbol{r}$ \nodepart[text width=1cm]{second} $\{ A  \} $ \nodepart{third} \phantom{\xmark}  };
				
				\node[treenode, below=0.5cm of eps2] (eps3) {$\boldsymbol{\epsilon}$ \nodepart[text width=1cm]{second} $\{ A,B  \} $ \nodepart{third} \cmark  };
				\node[treenode,right=0.2cm of eps3] (r3) {$\boldsymbol{r}$ \nodepart[text width=1cm]{second} $\{ A  \} $ \nodepart{third} \phantom{\xmark}  };
				\node[treenode,right=0.2cm of r3] (s3) {$\boldsymbol{s}$ \nodepart[text width=1cm]{second} $\{ K  \} $ \nodepart{third} \phantom{\xmark}  };
				
				\node[treenode, below=0.5cm of eps3] (eps4) {$\boldsymbol{\epsilon}$ \nodepart[text width=1cm]{second} $\{ A,B  \} $ \nodepart{third} \cmark  };
				\node[treenode,right=0.2cm of eps4] (r4) {$\boldsymbol{r}$ \nodepart[text width=1cm]{second} $\{ A  \} $ \nodepart{third} \phantom{\xmark}  };
				\node[treenode,right=0.2cm of r4] (s4) {$\boldsymbol{s}$ \nodepart[text width=1cm]{second} $\{ K  \} $ \nodepart{third} \phantom{\xmark}  };
				\node[treenode,right=0.2cm of s4] (rr4) {$\boldsymbol{rr}$ \nodepart[text width=1cm]{second} $\{ A \}$ \nodepart{third} \xmark };
				
				\node[treenode, below=0.5cm of eps4] (eps5) {$\boldsymbol{\epsilon}$ \nodepart[text width=1cm]{second} $\{ A,B  \} $ \nodepart{third} \cmark  };
				\node[treenode,right=0.2cm of eps5] (r5) {$\boldsymbol{r}$ \nodepart[text width=1cm]{second} $\{ A, B  \} $ \nodepart{third} \xmark  };
				\node[treenode,right=0.2cm of r5] (s5) {$\boldsymbol{s}$ \nodepart[text width=1cm]{second} $\{ K  \} $ \nodepart{third} \phantom{\xmark}  };
				\node[treenode,right=0.2cm of s5] (rr5) {$\boldsymbol{rr}$ \nodepart[text width=1cm]{second} $\{ A \}$ \nodepart{third} \xmark };
				
				\node[treenode, below=0.5cm of eps5] (eps6) {$\boldsymbol{\epsilon}$ \nodepart[text width=1cm]{second} $\{ A,B  \} $ \nodepart{third} \cmark  };
				\node[treenode,right=0.2cm of eps6] (r6) {$\boldsymbol{r}$ \nodepart[text width=1cm]{second} $\{ A, B  \} $ \nodepart{third} \xmark  };
				\node[treenode,right=0.2cm of r6] (s6) {$\boldsymbol{s}$ \nodepart[text width=1cm]{second} $\{ K  \} $ \nodepart{third} \cmark  };
				\node[treenode,right=0.2cm of s6] (rr6) {$\boldsymbol{rr}$ \nodepart[text width=1cm]{second} $\{ A \}$ \nodepart{third} \xmark };
				\node[treenode,right=0.2cm of rr6] (ss6) {$\boldsymbol{ss}$ \nodepart[text width=1cm]{second} $\{ A \}$ \nodepart{third} \phantom{\xmark} };
				
				\node[treenode, below=0.5cm of eps6] (eps7) {$\boldsymbol{\epsilon}$ \nodepart[text width=1cm]{second} $\{ A,B  \} $ \nodepart{third} \cmark  };
				\node[treenode,right=0.2cm of eps7] (r7) {$\boldsymbol{r}$ \nodepart[text width=1cm]{second} $\{ A, B  \} $ \nodepart{third} \xmark  };
				\node[treenode,right=0.2cm of r7] (s7) {$\boldsymbol{s}$ \nodepart[text width=1cm]{second} $\{ K  \} $ \nodepart{third} \phantom{\xmark}  };
				\node[treenode,right=0.2cm of s7] (rr7) {$\boldsymbol{rr}$ \nodepart[text width=1cm]{second} $\{ A \}$ \nodepart{third} \xmark };
				\node[treenode,right=0.2cm of rr7] (ss7) {$\boldsymbol{ss}$ \nodepart[text width=1cm]{second} $\{ A,B \}$ \nodepart{third} \xmark };
				
				\node[treenode, below=0.5cm of eps7] (eps8) {$\boldsymbol{\epsilon}$ \nodepart[text width=1cm]{second} $\{ A,B  \} $ \nodepart{third} \phantom{\xmark}  };
				\node[treenode,right=0.2cm of eps8] (r8) {$\boldsymbol{r}$ \nodepart[text width=1cm]{second} $\{ A, B  \} $ \nodepart{third} \xmark  };
				\node[treenode,right=0.2cm of r8] (s8) {$\boldsymbol{s}$ \nodepart[text width=1cm]{second} $\{ K, L  \} $ \nodepart{third} \cmark  };
				\node[treenode,right=0.2cm of s8] (rr8) {$\boldsymbol{rr}$ \nodepart[text width=1cm]{second} $\{ A \}$ \nodepart{third} \xmark };
				\node[treenode,right=0.2cm of rr8] (ss8) {$\boldsymbol{ss}$ \nodepart[text width=1cm]{second} $\{ A, B \}$ \nodepart{third} \xmark };
				
				\node[treenode, below=0.5cm of eps8] (eps9) {$\boldsymbol{\epsilon}$ \nodepart[text width=1.4cm]{second} $\{ A,B,M  \} $ \nodepart{third} \cmark  };
				\node[treenode,right=0.2cm of eps9] (r9) {$\boldsymbol{r}$ \nodepart[text width=1cm]{second} $\{ A, B  \} $ \nodepart{third} \phantom{\xmark}  };
				\node[treenode,right=0.2cm of r9] (s9) {$\boldsymbol{s}$ \nodepart[text width=1cm]{second} $\{ K, L  \} $ \nodepart{third} \cmark  };
				\node[treenode,right=0.2cm of s9] (rr9) {$\boldsymbol{rr}$ \nodepart[text width=1cm]{second} $\{ A \}$ \nodepart{third} \phantom{\xmark} };
				\node[treenode,right=0.2cm of rr9] (ss9) {$\boldsymbol{ss}$ \nodepart[text width=1cm]{second} $\{ A, B \}$ \nodepart{third} \xmark };
				
				\node[blanknode, right=0.4cm of eps0] {$\completionstep{A \subsumed \forall r. A}{\boldsymbol{\epsilon}}$};
				\node[blanknode, right=0.4cm of r1] {$\completionstep{A \subsumed B}{\boldsymbol{\epsilon}}$};
				\node[blanknode, right=0.4cm of r2] {$\completionstep{A \subsumed \forall s. K}{\boldsymbol{\epsilon}}$};
				\node[blanknode, right=0.4cm of s3] {$\completionstep{A \subsumed \forall r. A}{\boldsymbol{r}}$};
				\node[blanknode, right=0.4cm of rr4] {$\completionstep{A \subsumed B}{\boldsymbol{r}}$};
				\node[blanknode, right=0.4cm of rr5] {$\completionstep{K \subsumed \forall s. A}{\boldsymbol{s}}$};
				\node[blanknode, right=0.4cm of ss6] {$\completionstep{A \subsumed B}{\boldsymbol{ss}}$};
				\node[blanknode, right=0.4cm of ss7] {$\completionstep{\forall s. B \subsumed L}{\boldsymbol{s}}$};
				\node[blanknode, right=0.4cm of ss8] {$\completionstep{\forall s. L \subsumed M}{\boldsymbol{\epsilon}}$};
						
				\node[blanknode,left=0.4cm of eps0] {$\Y_0$};
				\node[blanknode,left=0.4cm of eps1] {$\Y_1$};
				\node[blanknode,left=0.4cm of eps2] {$\Y_2$};
				\node[blanknode,left=0.4cm of eps3] {$\Y_3$};
				\node[blanknode,left=0.4cm of eps4] {$\Y_4$};
				\node[blanknode,left=0.4cm of eps5] {$\Y_5$};
				\node[blanknode,left=0.4cm of eps6] {$\Y_6$};
				\node[blanknode,left=0.4cm of eps7] {$\Y_7$};
				\node[blanknode,left=0.4cm of eps8] {$\Y_8$};
				\node[blanknode,left=0.4cm of eps9] {$\Y_9$};
				\end{tikzpicture}
		}
		\caption{Example run \label{fig:examplerun}}
	\end{figure}
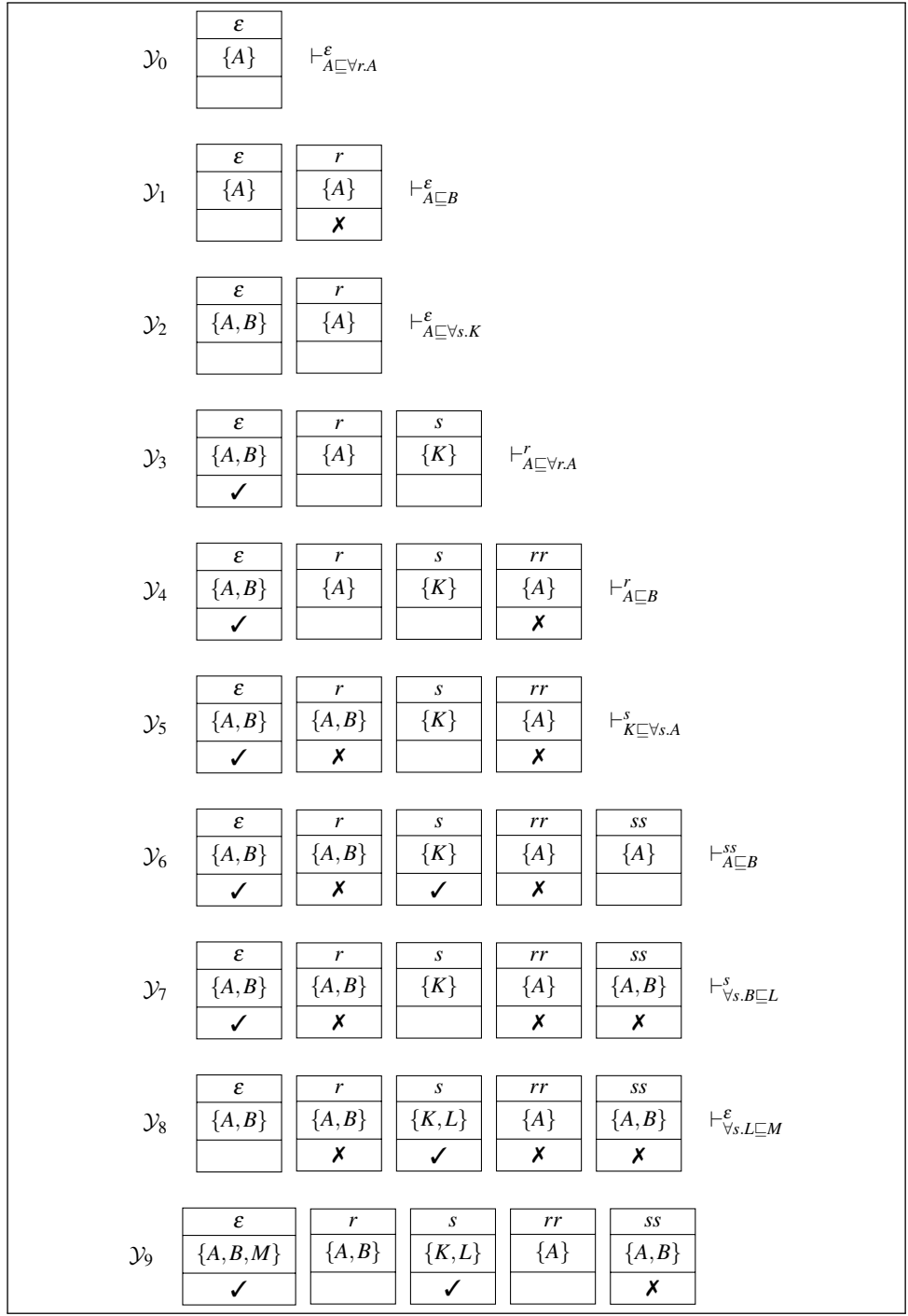
	We use a total order on $\NR^*$ that satisfies
	\begin{align*}
	\epsilon \prec r \prec s \prec rr \prec r s \prec s r \prec s s \prec r r r \prec \cdots
	\end{align*}
	and compute a sequence of completion steps for \SUBSpar{A}{M}{\T} sketched in Figure \ref{fig:examplerun}, where  
	\begin{itemize}
		\item[\xmark] marks blocked elements, and
		\item[\cmark] marks non-blocked elements not violating any GCI in \T. 
	\end{itemize}
	We write $\completionstep{A \subsumed \forall r. A}{\boldsymbol{\epsilon}}$ to denote the completion step that takes 
	$\boldsymbol{\epsilon}$ as a non-blocked element violating $A \subsumed \forall r. A$ and applies the expansion.
	Figure \ref{fig:examplerun} shows the first completion steps needed to obtain $M \in \Y_9(\epsilon)$, which yields $A \subsumedT M$. 
	For example, in $\Y_1$ the blocking condition \bRule{2} is used to block the node $\boldsymbol{r}$. In 
        $\Y_2$, $\boldsymbol{r}$ is no longer blocked since the label of $\boldsymbol{\epsilon}$ has been expanded. In
	$\Y_5$,  $\boldsymbol{r}$ gets again blocked since its label is expanded, and thus
        $\boldsymbol{rr}$ is indirectly blocked due to \bRule{3}. Also note that in $\Y_6$ we have $\boldsymbol{rr} \prec \boldsymbol{ss}$ 
	and both have the same label, but since $\boldsymbol{rr}$ is already blocked, \bRule{2} does not apply to $\boldsymbol{ss}$, which allows 
	us to do further completion steps needed to derive $A \subsumedT M$.
\end{example}	

Before we prove that the algorithm is sound and complete, we first show that the computed sequence is always finite, thus ensuring termination of the algorithm.
The \emph{depth} of a partial functional interpretation $\Y = (\domainof{\Y},\cdot^\Y)$, denoted by $\depthtreeof{\Y}$, is the maximum length of role words in $\domainof{\Y}$, i.e.,  $\depthtreeof{\Y} := \mathsf{max}(\{ \abs{\sigma} \mid \sigma \in \domainof{\Y}  \})$.

\begin{lemma}\label{lem:depthbound}
	If $\Z$ is a partial functional interpretation such that $\Y_0 \completionstepT^* \Z$, then
	$
	\depthtreeof{\Z} \leq 2^\abs{\NC} + 1.
	$
\end{lemma}
\begin{proof}
Let $\Y_0\completionstepT \Y_1 \completionstepT\ldots\completionstepT \Y_n=\Z$ be a sequence of expansions. 
We show for each $i$, $1\leq i\leq n$, that the length of words in $\domainof{\Y_i}$ is bounded by $2^\abs{\NC}+1$. 
A new element $\sigma\in\domainof{\Y_i}\setminus\domainof{\Y_{i-1}}$ is only added by the expansion at $\sigma$ of $\Y_{i-1}$ if 
$\sigma=\omega r$ and $\omega\in\notblocked{\Y_{i-1}}$. Now, $\omega\in\notblocked{\Y_{i-1}}$ is only possible if there exist 
no two distinct $\sigma_1,\sigma_2\in\prefixset{\omega}$ such that $\Ymc_{i-1}(\sigma_1)=\Ymc_{i-1}(\sigma_2)$. 
Otherwise, since either $\sigma_1\prec\sigma_2$ or $\sigma_2\prec\sigma_1$, one of these two nodes would be blocked by blocking condition~\bRule{2}, 
and $\omega$ would be blocked by condition~\bRule{3}. 
It follows that $\Y_{i-1}(\sigma_1)\neq\Y_{i-1}(\sigma_2)$ for every two distinct $\sigma_1$, $\sigma_2\in\prefixset{\omega}$, 
and consequently $\abs{\omega}\leq 2^{\abs{\NC}}$ and $\abs{\sigma}=\abs{\omega}+1\leq 2^\abs{\NC}+1$. 
Hence, $\abs{\sigma}\leq 2^{\abs{\NC}}+1$  for every $\sigma\in\domainof{\Z}$, which yields $\depthtreeof{\Z}\leq 2^\abs{\NC}+1$.
 \hfill \end{proof}

The upper bound on the depth of the tree in a \T-completion sequence also yields an upper bound on its overall size, 
since the outdegree of the tree is limited by $\big|\NR\big|$. 
Furthermore, we observe that $\Y \completionstepT \Y'$ implies that $\Y \subsetneq \Y'$, \ie a \T-completion always adds something 
and never removes anything. 
At the same time, each label set can  contain at most $\big| \NC \big|$ many names.
Thus, due to the depth bound, the bound on the outdegree,  and the upper bound on the label size, there cannot be an infinite sequence of \T-completions. Hence, $\SUBSpar{A_0}{B_0}{\T}$ always terminates. Note that we have used both blocking conditions, \bRule{2} and \bRule{3}, in the proof.

\begin{lemma}
$\SUBSpar{A_0}{B_0}{\T}$ always terminates.
\end{lemma}

Note, however, that our termination argument only yields a double-exponential bound on the run time of the algorithm. 
The reason is that Lemma~\ref{lem:depthbound} only shows an exponential bound on the \emph{depth} of the generated trees,
and thus only a double-exponential bound on the \emph{size} of these trees. At the moment, it is not clear
whether one can construct examples where the algorithm only terminates after an double-exponential number of steps,
but we also do not have a proof that it always terminates in exponential time. 
Thus, we currently do not know whether the algorithm is worst-case optimal or not. 
However, our experimental evaluation shows that it works reasonably well in practice.

It remains to show that $\SUBSpar{A_0}{B_0}{\T}$  always computes the correct result, i.e., that it is sound and complete. 
The following lemma is crucial for proving this.

\begin{lemma}\label{lem:relationCanonicalModel}
 Let $\Y_0$ be as in~\eqref{eq:inittree} and $\Y$ be a partial functional interpretation that is reachable from $\Y_0$ and complete,
that is, $\Y_0\completionstepT^*\Y$ and $\incomp{\Y,\T}\cap\notblocked{\Y}=\emptyset$.
Then there is a functional model $\I$ of $\T$ such that $\Y(\epsilon)=\I(\epsilon)$.
\end{lemma}
\begin{proof}
We extend $\Y$ to a functional interpretation $\I$ such that\ $\Y(\epsilon)=\I(\epsilon)$.
Note that, in $\Y$, even non-blocked nodes $\sigma$ need not have $r$-successors for all $r\in \NR$. This is the case if there is no GCI 
that requires generating an $r$-successor for $\sigma$.  In the least functional model, the successor $\sigma r$ exists, but it has label $\emptyset$.
We will represent such successors by a dummy node $d_\top$ with an empty label in our construction.

To construct $\I$, we first define a mapping $m:\NRstar\rightarrow\notblocked{\Y}\cup\{d_\top\}$ by induction on the length of $\sigma\in\NRstar$ as follows: 
\\[-\medskipamount] \mbox{ } \hfill \parbox[t]{0.935\textwidth}{
	\begin{itemize}
         \item By definition, $\epsilon$ is not blocked, and thus we can set $m(\epsilon) = \epsilon$.
         \item Now, consider a node $\sigma r$ of length $> 0$,  and assume that $m(\sigma)$ is already defined.
               We distinguish two cases:
         \begin{itemize}
          \item
               Assume that $m(\sigma) r\in\domainof{\Y}$. Note that this node cannot be indirectly blocked since
               $m(\sigma)$ is then a node in $\domainof{\Y}$ that is not blocked.
               Thus, 
               there exists $\sigma'\in\notblocked{\Y}$ such that $\Y(\sigma')=\Y(m(\sigma) r)$. 
               We set $m(\sigma r)=\sigma'$.
	 
	 \item If $m(\sigma) r\not\in\domainof{\Y}$, then we set $m(\sigma r)=d_\top$.
        \end{itemize}
        \end{itemize}
}
\smallskip \noindent
Based on $m$ and $\Y$, we define the functional interpretation $\I$ by setting 
$$A^\I=\{\sigma\mid m(\sigma)\in A^\Y\}\ \ \ \mbox{for all $A\in\NC$.}$$
It follows from Definition~\ref{def:match} that, for every $\sigma\in\NRstar$ and every \FLnull concept $C$ in normal form, 
if $m(\sigma)$ matches $C$ in $\Y$, then $\sigma\in C^\I$. In fact, assume that $m(\sigma)$ matches $C$.
If $A$ is a conjunct in $C$, then $m(\sigma)\in A^\Y$, and thus
$\sigma\in A^\I$. If $\forall r.A$ is a conjunct in $C$, then $m(\sigma) r\in A^\Y$. This implies $m(\sigma) r\in\domainof{\Y}$,
and thus $m(\sigma r)$ satisfies $\Y(m(\sigma r))=\Y(m(\sigma) r)$, which yields $m(\sigma r)\in A^\Y$, and thus $\sigma r\in A^\I$.
This shows $\sigma\in (\forall r.A)^\I$. 

The other direction also holds. Assume that $\sigma\in C^\I$. If $A$ is a conjunct in $C$, then 
$\sigma\in A^\I$ implies $m(\sigma)\in A^\Y$. If $\forall r.A$ is a conjunct in $C$, then $\sigma\in (\forall r.A)^\I$ 
implies $\sigma r\in A^\I$, and thus $m(\sigma r)\in A^\Y$. Consequently,
$A\in \Y(m(\sigma r)) = \Y(m(\sigma)r)$ yields $m(\sigma)r\in A^\Y$, which completes the proof that $m(\sigma)$ matches $C$

We are now ready to show that $\I$ is a model of $\T$, that is, for every $C\sqsubseteq D\in\T$ and $\sigma\in C^\I$, also $\sigma\in D^\I$ holds.
Thus, assume $C\sqsubseteq D\in\T$ and $\sigma\in C^\I$. The latter implies that $m(\sigma)$ matches $C$. This is only possible if
$m(\sigma)\neq d_\top$. Thus, $m(\sigma)\in\notblocked{\Y}$ and since $\Y$ is complete, 
$m(\sigma)\not\in\incomp{\Y}$. Consequently, $m(\sigma)$ matches $D$, which yields $\sigma\in D^\I$. 
 \hfill \end{proof}

\begin{theorem}
 $\SUBSpar{A_0}{B_0}{\T}$ is sound and complete, that is, it outputs \enquote{yes} iff $A_0\sqsubseteq_\T B_0$.
\end{theorem}

\begin{proof}
Assume that the algorithm has generated a complete partial functional interpretation $\Y$ such that  $\Y_0\completionstepT^*\Y$.
Lemma~\ref{lem:relationCanonicalModel} yields a model $\I$ of $\T$ such that $\I(\epsilon) = \Y(\epsilon)$.

If $\SUBSpar{A_0}{B_0}{\T}$ outputs \enquote{no}, then $B_0\not\in \Y(\epsilon)$. 
Since $A_0\in \Y(\epsilon) = \I(\epsilon)$ and $B_0\not\in \Y(\epsilon) = \I(\epsilon)$, the model $\I$ of $\T$ yields a counterexample
to the subsumption relation $A_0\sqsubseteq_\T B_0$  because this implies $\epsilon\in A_0^\I\setminus B_0^\I$.

If $\SUBSpar{A_0}{B_0}{\T}$ outputs \enquote{yes}, then $B_0\in \Y(\epsilon)$. It is easy to see that $Y(\sigma) \subseteq \canmodAnullT(\sigma)$
holds for all $\sigma\in\NR^*$. In fact, one can generate $\canmodAnullT$ from $\Y_0$ by an infinite number of completion steps that also
are applied to blocked nodes. Thus, whatever is added in the sequence $\Y_0\completionstepT^*\Y$ is also present in
$\canmodAnullT$. But then $B_0\in \Y(\epsilon)$ yields $B_0\in \canmodAnullT(\epsilon)$, and this implies  $A_0\sqsubseteq_\T B_0$ by Theorem~\ref{sub:char:least-functional}.
 \hfill \end{proof}

The algorithm $\SUBSpar{A_0}{B_0}{\T}$ shares properties with the completion method for \EL \cite{BaBL05} as well as with  
tableau algorithms for expressive DLs \cite{DBLP:journals/sLogica/BaaderS01}. Every single \T-completion step extends the label set of 
at least one node in the tree. Intuitively, adding the concept name $A$ to the label set of domain element $\sigma$ corresponds to deriving $A_0 \subsumed \forall \sigma. A$ as a consequence of $\T$. A single run of $\SUBSpar{A_0}{B_0}{\T}$ not only decides whether $A_0 \subsumed B_0$ is entailed by $\T$ but computes \emph{all} named subsumers of $A_0$. This is similar to the \EL completion method and other consequence-based calculi \cite{DBLP:conf/ijcai/SimancikKH11}. From tableau algorithms $\SUBSpar{A_0}{B_0}{\T}$ inherits the blocking mechanism that ensures termination.

\section{Horn and other fragments of \FLnull}\label{sect:horn}

\newcommand{\HornFLnull}{\text{Horn-}\FLnull}
\newcommand{\HornFLnullBot}{\text{Horn-}\FLnullBot}

\newcommand{\tup}[1]{(#1)}

Based on the algorithm presented in the last section, we show that subsumption 
between \FLnull concepts becomes tractable if one restricts to the 
Horn logic \HornFLnull introduced in~\cite{DBLP:conf/aaai/KrotzschRH07}. We then consider some extensions.
In \HornFLnull, every GCI is of one of the following forms:
\begin{gather}
 A\sqsubseteq C \quad A\sqcap B\sqsubseteq C \quad A\sqsusbeteq\forall r.B,
 \label{eq:horn-flnull-axioms}
\end{gather}
where $A,B,C\in\NC$ and $r\in\NR$. Our definition differs slightly from that 
in~\cite{DBLP:conf/aaai/KrotzschRH07}, in that they allow $\top$ and 
$\bot$ to be used both in $\FLnull$ and $\HornFLnull$. To see that this is not 
a major restriction, we note that for the extension of $\HornFLnull$ that uses 
$\top$ and $\bot$ anywhere where a concept is used, the reduction presented 
in Section~\ref{ssec:reductions} can still be used to obtain a TBox fully in 
$\HornFLnull$ as it is presented here.

\cite{DBLP:conf/aaai/KrotzschRH07} only show the complexity for knowledge base 
consistency, which is 
\PTime-complete in \HornFLnull. We improve upon these results by showing that 
subsumption between arbitrary \FLnull concepts with respect to a \HornFLnull 
TBox is tractable as well. Note that, whereas for \FLnull, subsumption between 
concepts can be reduced to knowledge base consistency, the restricted 
expressivity of \HornFLnull does not allow for this in the general case. 

\newcommand{\HornSROIQ}{\text{Horn-}\ensuremath{\mathcal{SROIQ}}\xspace}
\newcommand{\ALCI}{\ensuremath{\mathcal{ALCI}}\xspace}

\begin{theorem}\label{the:horn-fl0}
 Concept subsumption of \FLnull concepts with respect to general \HornFLnull TBoxes is \PTime-complete.
\end{theorem}
\begin{proof}
 {Hardness follows easily from \PTime-hardness of satisfiability of 
 propositional Horn formulae. Specifically, given a Horn formulae $\Phi$ over propositional variables $\{p_1,\ldots,p_m\}$, we associate to each variable $p_i$ a concept name $A_i$, translate clauses $p_{i_1}\wedge\ldots\wedge p_{i_m}\rightarrow p_j$ to GCIs $A_0\sqcap A_{i_1}\sqcap\ldots\sqcap A_{i_m}\sqsubseteq A_j$, and clauses
 $p_{i_1}\wedge\ldots\wedge p_{i_m}\rightarrow\bot$ to $A_0\sqcap A_{i_1}\sqcap\ldots A_{i_m}\sqsubseteq B_0$. Then, we transform these GCIs into ones
with only binary conjunction on the left-hand sides by introducing auxiliary concept names. It is easy to see that the resulting TBox entails $A_0\sqsubseteq B_0$ iff $\Phi$ is unsatisfiable.}

 For inclusion in \PTime, we modify the procedure described in Section~\ref{sec:subs}.
In contrast to that procedure, we cannot 
reduce 
subsumption of the form $C\sqsubseteq_\T D$ to subsumptions of the form 
$A_0\sqsubseteq_\T B_0$, since the axiom $D\sqsubseteq B_0$ need not be expressible in \HornFLnull. 
However, we can restrict ourselves to 
subsumptions of the form $A_0\sqsubseteq_\T D$, where $A_0\in\NC$, as for 
subsumptions $C\sqsubseteq D$, we can add the axiom $A_0\sqsubseteq C$ to the 
original TBox, which after normalization becomes an $\HornFLnull$ TBox $\T$
{that entails $A_0\sqsubseteq D$ iff the original ontology entails $C\sqsubseteq D$.}

To decide $A_0\sqsubseteq_\T D$ in polynomial time, we apply 
the algorithm described in Section~\ref{sec:subs} with two modifications:
\\[-\medskipamount] \mbox{ } \hfill 		
	\parbox[t]{0.935\textwidth}{
	\begin{enumerate}
	 \item the initial partial functional interpretation $\Y_0$ already contains several nodes which 
	serve 
	as a \enquote{skeleton} of $D$, and 
	 \item expansions are only applied on nodes from that skeleton. 
	\end{enumerate}
	} 

\smallskip \noindent
\noindent
Specifically, for $D=\forall \sigma_1.A_1\sqcap\ldots\sqcap\forall 
\sigma_n.A_n$, the 
initial partial functional interpretation $\Y_0$ is now defined as 
follows:
\[
\Delta^{\Y_0}=\bigcup_{1\leq i\leq n}\prefixset{\sigma_i}
\qquad
  A_0^{\Y_0}=\{\epsilon\} \qquad B^{\Y_0}=\emptyset \text{ for all } 
B\in\NC\setminus\{A_0\}.
\]
Furthermore, expansions are only applied on nodes $\sigma\in\Delta^{\Y_0}$, 
that is, 
new nodes may be introduced, but they are not further expanded. This 
restriction makes every completion sequence polynomially bounded{, because} we have 
at most one step per pair $(\alpha, \sigma)\in\T\times\Delta^{\Y_0}$. For the 
final 
interpretation $\Z$, we check whether {$\sigma_i\in A_i^\Z$} for all $1\leq i\leq n$, 
which corresponds to checking whether $\epsilon\in D^\Z$. To show that the 
resulting method is still sound and complete, we show that for the 
least functional model $\I_{A,\T}$, we have for every 
$\sigma\in\domainof{\Y_0}$ that {$\Z(\sigma)=I_{A,\T}(\sigma)$}.
For this, it suffices to show that, for every $d\in\Y^0$ and $C'\sqsubseteq 
D'\in\T$, {$\sigma\in (C')^\Z$} implies {$\sigma\in (D')^\Z$}. Since $\T$ is in 
\HornFLnull, 
$C'$ 
does not contain universal role restrictions. Consequently, if 
$\sigma\in\match{C'}{\Z}$, the expansion {already made} sure that 
$\sigma\in\match{D'}{\Z}$ and consequently that $\sigma\in (D')^\Z$. It follows 
that 
$\Z(\sigma)=I_{A,\T}(\sigma)$ for all $\sigma\in\domainof{\Y_0}$. This means 
that $A\sqsubseteq_\T D$ iff 
$\epsilon\in D^\Z$. Our method runs in polynomial time and is sound 
and complete, and thus subsumption with \HornFLnull-TBoxes can be 
decided in polynomial time. \hfill
\end{proof}

\begin{remark}\label{rem:horn}
The proof of Theorem~\ref{the:horn-fl0} uses the fact that we only need to 
consider role-successors of roles that occur on the left-hand side of a GCI 
(in case of \HornFLnull there are no such roles to consider). We use this 
observation in an optimization of \flower to improve reasoning times.
\end{remark}

\newcommand{\HornFLnullBotPlus}{\ensuremath{\text{Horn-}\mathcal{FL}_\bot^+}\xspace}

\newcommand{\HornFLminus}{Horn-$\mathcal{FL}^-$\xspace}

{For many DLs, such as \ALC and \ALCI, it is common to define their Horn-fragments as their intersection with \HornSROIQ. If we define \HornFLnullBot in this way, we obtain a DL in which value restrictions can occur on the left-hand side in axioms of the form  
$A\sqcap\forall r.B\sqsubseteq\bot$, where $A,B\in\NC$ and $r\in\NR$. Specifically, in}
%
\HornFLnullBot, every axiom is of the form
\begin{align}
 A\sqsubseteq B\qquad A\sqcap B\sqsubseteq C \qquad
 A\sqsubseteq\forall r.A \qquad A\sqcap \forall r.B\sqsubseteq\bot,
 \label{eq:horn-flnullbot-axioms}
\end{align}
where {$A,B,C\in\NC\cup\{\top,\bot\}$} and $r\in\NR$. 

\begin{theorem}
 Subsumption between concept names is \PSpace-complete for \HornFLnullBot. 
\end{theorem}
\begin{proof}
 Both directions can be shown by showing a relation to \HornFLminus, for which subsumption between concept names is also \PSpace-complete~\cite{DBLP:conf/aaai/KrotzschRH07}.
 \HornFLminus is similar to \HornFLnullBot, but instead of axioms of 
the form $A\sqcap\forall r.B\sqsubseteq\bot$, it allows {for} axioms of the form 
$A\sqsubseteq\exists r$, where the {semantics} of $\exists r$ is defined by 
$(\exists r)^\Imc=\{d\mid \exists e\in\Delta^\Imc, (d,e)\in r^\Imc\}$. The \HornFLminus axiom $A\sqsubseteq\exists r$ is equivalent to the \HornFLnullBot axiom $A\sqcap\forall r.\bot\sqsubseteq\bot$, which means every \HornFLminus ontology can be easily translated into \HornFLnullBot. 
This establishes \PSpace-hardness of \HornFLnullBot. 

For inclusion in \PSpace, we show how every \HornFLnullBot ontology can be translated in polynomial time into a \HornFLminus ontology. For this, we replace every axiom {$\alpha$} of the 
form $A\sqcap\forall r.B\sqsubseteq\bot$ by the axioms $A\sqsubseteq\exists r_\alpha$, 
$A\sqsubseteq\forall r_\alpha.\overline{B}$ and $B\sqcap\overline{B}\sqsubseteq\bot$, where $r_\alpha$ is fresh for every such axiom {$\alpha$}. In addition, for every such fresh introduced role $r_\alpha$ and every axiom of the form $C\sqsubseteq\forall r.D$, we add $C\sqsubseteq\forall r_\alpha.D$. 
{Intuitively, $A\sqcap\forall r.B\sqsubseteq\bot$ is satisfied iff every instance of $A$ has some $r$-successor that does not satisfy $B$. As there may be several such axioms, we need to distinguish between different $r$-successors for each such axiom. \HornFLminus is not expressive enough to do that directly, which is why we use a different role for every such axiom.
}

Let $\T$ be the TBox before this transformation and $\T'$ the result, and $A$, $B$ be two concept names occurring in $\T$. We show that $\T\models A\sqsubseteq B$ iff 
{$\T'\models A\sqsubseteq B$}.

{($\Rightarrow$)
Assume $\T'\not\models A\sqsubseteq B$, which means there exists some model $\Imc'$ of $\T'$ s.t. $\Imc'\not\models A\sqsubseteq B$. We construct a model $\Imc$ of $\T$ s.t. $\Imc\not\models A\sqsubseteq B$ by setting $\Delta^\Imc=\Delta^{\Imc'}$, $A^\Imc=A^{\Imc'}$ for all $A\in\NC$, and 
}
\comment{If for 
some model {$\Imc'$} of $\T'$, {$\Imc'\not\models A\sqsubseteq B$}, we 
transform {$\Imc'$} into a model {$\Imc$} of $\T$ by setting }
$${r^{\Imc}=r^{\Imc'}\cup\bigcup_{\alpha = (A'\sqcap\forall r.B'\sqsubseteq\bot) \in \T}r_\alpha^{\Imc'}}$$
{for all $r\in\NR$.}
{For} every introduced role name $r_\alpha$ and every axiom {$A'\sqsubseteq\forall r.B'\in\T$}, we have {$\Imc'\models A'\sqsubseteq\forall r_\alpha.B'$}, which yields {$\Imc\models A'\sqsubseteq\forall r.B'$}. Furthermore, for every {$\alpha=A'\sqcap\forall r.B'\sqsubseteq\bot\in\T$} and {$d\in (A')^{\Imc}$}, there exists {$(d,e)\in\ r^{\Imc'}$} s.t. {$(d,e)\in r_\alpha^{\Imc'}$}{ and } {$e\in(\overline{B'})^{\Imc'}$,} {which implies {$e\not\in (B')^{\Imc}$} and $d\not\in(\forall r.B')^{\Imc}$.} 
Thus, we have show that {$\Imc$} is a model of $\T$  and that
{$\Imc\not\models A\sqsubseteq B$}, and thus $\T\not\models A\sqsubseteq B$. 

{($\Leftarrow$)} Now let $\Imc$ be a model of $\T$ s.t. $\Imc\not\models A\sqsubseteq B$. \comment{We 
extend $\Imc$ to a model of $\T'$ by specifying the interpretation of the fresh 
role names. }
{Based on $\Imc$, we construct a model $\Imc'$ of $\T'$ s.t. $\Imc'\not\models A\sqsubseteq B$.}
For every $\alpha=A\sqcap\forall r.B\sqsubseteq\bot\in\T$ and $d\in 
A^\Imc$, there exists some $e\in\Delta^\Imc$ s.t. $(d,e)\in r^\Imc$ and 
$e\not\in B^\Imc$. The interpretation $r_\alpha^{\Imc'}$ of the role $r_\alpha$ is defined 
\comment{by collecting }
{as the set of }
all those pairs 
$(d,e)$. {All other concept and role names are interpreted as in $\Imc$.}
The resulting interpretation $\Imc'$ satisfies all axioms in $\T'$ 
and thus $\T'\not\models A\sqsubseteq B$. 

{Summing up, we have shown that }%
\comment{We obtain that }%
$\T\not\models 
A\sqsubseteq B$ iff $\T'\not\models A\sqsubseteq B$, and thus that subsumption 
between concept names in $\HornFLnullBot$ can be polynomially reduced to 
subsumption between concept names in \HornFLminus. 
\end{proof}

We have used a modification of the algorithm presented in Section~\ref{sec:subs} 
to show that subsumption in \HornFLnull is \PTime-complete, thus indicating optimality of our algorithm
for this fragment. To deal with $\bot$, we could try to employ the reduction presented
in~Section~\ref{ssec:reductions}, which introduces a concept name for $\bot$.
Unfortunately, this approach cannot 
work for \HornFLnullBot. 
In fact, 
if we generalized axioms of the form $A\sqcap\forall r.B\sqsubseteq\bot$ to ones 
that use a concept name instead of $\bot$, we would have to allow axioms of the 
form $A\sqcap\forall r.B\sqsubseteq C$. This makes the logic powerful enough to 
cover the whole language of $\FLnull$, as we can represent 
axioms of the form $A\sqcap\forall r.B_1\sqcap\forall s.B_2\sqsubseteq C$ using 
$A\sqcap\forall r.B_1\sqsubseteq D$ and $D\sqcap\forall r.B_2\sqsubseteq C$, 
and axioms of the form $A\sqcap\forall r.B\sqsubseteq\forall s.C$ using 
$A\sqcap\forall r.B\sqsubseteq D$, $D\sqsubseteq\forall r.C$, where in each 
case, $D$ is fresh. 
In fact, already allowing more than one value restriction on the left-hand increases the complexity. 

If we further relax \HornFLnullBot to allow several value restrictions on the 
left-hand side, the logic becomes again \ExpTime-complete. In 
\HornFLnullBotPlus, axioms are of the 
forms listed in~\eqref{eq:horn-flnullbot-axioms} and the following form:
\begin{align}
 \forall \sigma_1.A_1\sqcap\ldots\sqcap\forall \sigma_n.A_n\sqsubseteq\bot,
\label{eq:horn-fl0bot-plus}
\end{align}
where for $1\leq i\leq n$, $\sigma_i\in\NRstar$ and $A_i\in\NC$.

Hardness of \HornFLnullBotPlus can be shown based on the reduction used in the 
proof for Proposition~1 in~\cite{BaTh-KI-20} employed to show 
\ExpTime-hardness of \FLnull. The reduction uses a TBox that is not in 
\HornFLnullBotPlus and does not even contain $\bot$. However, it uses a special concept name $F$ which essentially mimics the behavior of $\bot$. Replacing $F$ by $\bot$ creates a \HornFLnullBotPlus TBox with a similar behavior.
Specifically, $F$ occurs on the right-hand side of the subsumption test, in 
axioms of the form $A\sqcap B\sqcap\forall w_1.F\sqcap\ldots\sqcap 
w_n.F\sqsubseteq F$ (Axiom 2), $A_1\sqcap\ldots\sqcap A_n\sqsubseteq\forall r.F$ 
(Axiom 7) and in axioms $F\sqsubseteq\forall r.F$, which are added for every 
role name $r$ used in the reduction (Axioms 8 and 9). All other axioms are in 
\HornFLnull. Thus, replacing $F$ by $\bot$ results in a TBox of the desired 
form. We argue that in the resulting TBox, $C\sqsubseteq\bot$ is entailed iff 
$C\sqsubseteq F$ is entailed in the original TBox, where $C$ does not contain 
$F$. If $C\sqsubseteq F$ is entailed by the original TBox, clearly 
$C\sqsubseteq\bot$ is entailed by the transformed. 

For the other direction, 
assume that $C\sqsubseteq F$ is not entailed by the original ontology, and let 
$\Imc$ be a witnessing model with $d\in C^\Imc\setminus F^\Imc$ such that every 
domain element is reachable by a path of role-successors from $d$. We transform 
$\Imc$ into $\Imc'$ by removing all elements in $F^\Imc$. Since $\Imc\models 
F\sqsubseteq\forall r.F$ for all $r\in\NR$, we have for all domain elements 
$e\in\Delta^{\Imc'}$ and words $w\in\NRstar$, $e\in(\forall w.F)^{\Imc}$ iff 
$e\in(\forall w.\bot)^{\Imc}$. It follows that for every axiom of the form 
$A\sqcap B\sqcap\forall w_1.F\sqcap\ldots\sqcap\forall w_n.F\sqsubseteq F$ in 
$\T$, $\Imc'\models A\sqcap\forall w_1.\bot\sqcap\ldots\sqcap\forall 
w_n.\bot\sqsubseteq\bot$, and for every axiom of the form $A_1\sqcap\ldots\sqcap 
A_n\sqsubseteq\forall r.F$, $\Imc'\models A_1\sqcap\ldots\sqcap 
A_n\sqsubseteq\forall r.\bot$. The axioms $\bot\squbseteq\forall r.\bot\in\T$ 
are naturally entailed. None of the remaining axioms have value restrictions on 
the left-hand side, and are thus also entailed by $\Imc'$. Consequently,
$\Imc'$ is a model of the transformed TBox. 

Thus, we have shown that the reduction used 
in~\cite{BaTh-KI-20} to show \ExpTime-hardness of \FLnull can be adapted to 
show \ExpTime-hardness of \HornFLnullBotPlus.

\begin{theorem}
 Deciding subsumption in $\HornFLnullBotPlus$ is \ExpTime-complete.
\end{theorem}

\section{A Rete-based Implementation}\label{sec:rete}

Our implementation of $\SUBSpar{A_0}{B_0}{\T}$ in \flower employs a variant of the algorithm for Rete networks~\cite{DBLP:journals/ai/Forgy82} to allow for a fast generation of completions of the partial model to be constructed. Specifically, the Rete network tests on
all domain elements 
satisfaction of all GCIs at the same time. It stores also partial matches so that they can be quickly continued once additional information is available. In addition, \flower  uses optimized data structures to allow for a fast and memory-efficient navigation in the current model, as well as to speed-up the implementation of blocking.

\subsection{Rete network for the TBox to speed-up  matching of GCIs}
\label{sec:Rete}

In order to compute a sequence of \T-completions $\Y_0 \completionstepT \Y_1 \completionstepT \Y_2 \completionstepT \cdots$ 
starting from the initial partial functional interpretation $\Y_0$, we can employ a GCI $C \subsumed D \in \T$ like a rule of the form 
\begin{align*}
?\sigma \in \match{C}{\Y_i} \rightarrow ~ ?\sigma \in \match{D}{\Y_i},	
\end{align*}
where $?\sigma$ ranges over the non-blocked domain elements of $\Y_i$, to obtain the next expansion. Overall, the rules corresponding to the GCIs from the TBox are applied during a run of $\SUBSpar{A_0}{B_0}{\T}$ in a forward-chaining manner to yield the  sequence of \T-completions.

In each expansion step $i$, one has to compute  the elements that violate a GCI, i.e., the pairs $(\sigma,C \subsumed D) \in \left(\domainof{\Y_i} \cap \notblocked{\Y_i} \right) \times \T$ such that 
$\sigma$ matches $C$ but not $D$ in $\Y_i$. 
Since there is potentially a large number of elements in 
$\domainof{\Y_i}$ that has to be matched against a large number of left-hand sides of GCIs (patterns) in the TBox in each step, we have chosen to implement this 
task using the Rete algorithm for many pattern/many object matching 
\cite{DBLP:journals/ai/Forgy82}, which is tailored to efficiently compute forward chaining rule applications. The general idea is to integrate the matching tests of all GCIs using a Rete network, which is in our case a compressed network-representation of the TBox. 
In each completion step, the extension of a tree $\Y_i$ only affects a small number of its elements: the matching element $\sigma$ itself and/or its children. This makes the Rete-based algorithm particularly efficient in our setting, because it stores matching information across completion steps to avoid reiterating over the whole set of pairs $\left(\domainof{\Y_i} \cap \notblocked{\Y_i} \right) \times \T$ in each step. Only those elements with changes have to be re-matched again in the next completion step.

For a given element, the network tests which left-hand sides of a GCI are matched and triggers the extension for the corresponding right-hand side. 
This Rete network corresponds to a graph using three kinds of nodes: a single root node, a set of intermediate nodes and a set of terminal nodes. Intuitively, the \emph{intermediate nodes} check for matches of parts of the left-hand side of a GCI, while the \emph{terminal nodes} hold the right-hand side of a GCI that is ready to be applied to an element. 
To process an element $\sigma\in\domainof{\Y}$, a set of so-called tokens is passed from the root node through the intermediate nodes to the terminal nodes. 
Such a \emph{token} is a pair of the form 
$(\sigma, r) \in \big(\NR^*, \NR \cup \{ \epsilon \}\big)$.
Intuitively, the token $(\sigma, \epsilon)$ is used 
to check whether $\sigma$ matches the concept names on the left-hand side of a GCI, 
while a token of the form $(\sigma, r)$ with $r \in \NRCT$ 
is used to check whether $\sigma$ matches value restrictions with the role name $r$. 

There are the following three types of \emph{intermediate nodes} that process tokens arriving from predecessor nodes in the network:
\begin{itemize}
	\item A \emph{concept node} is labeled with a concept name $B \in \NC$ and sends an incoming token $(\sigma, s)$ to all successor nodes
	iff $\sigma s \in B^{\Y_i}$.
	\item A \emph{role node} is labeled with an 
	$s \in \NR \cup \{ \epsilon \}$. 
	An arriving token of the form $(\sigma, s')$ is handled as follows. If $s \in \NRCT$ and $s'=s$, then it sends $(\sigma,s')$ to all successor nodes. If $s = \epsilon$, then it sends the token $(\sigma s', \epsilon)$ to all successor nodes.
	\item An \emph{inter-element node} is labeled with a tuple $(s_1,\ldots,s_m) \in (\NRCT \cup \{ \epsilon \})^m$. 
	It stores all arriving tokens and sends a token $(\sigma, \epsilon)$ to its successor nodes once \emph{all} tokens of the form $(\sigma, s_1),\ldots,(\sigma,s_m)$ have arrived at this node.
\end{itemize}
The overall network is structured in layers. The root node with no incoming edges is on top. 
All successors of the root node are concept nodes. The root node takes an element of the form $\sigma = \rho r \in \NRCT^*$ and sends the token $(\rho, r)$ to all successor nodes. A successor of a concept node can only be another concept node or a role node. A role node leads directly to an inter-element node and inter-element nodes lead to terminal nodes. Intuitively, paths of concept nodes corresponds to conjunctions of concept names a token must satisfy in order to pass through them. These concept names either need to be matched on the current element or on its immediate role successors. If the path of concept names goes into a role node labeled with $\epsilon$, this corresponds to a match on the current element. If it goes into a role node labeled with a role name $r$, this corresponds to a match on its $r$-successors. The inter-element nodes again correspond to a conjunction that combine the successful matches of the different role-successors.

\begin{example}\label{ex:rete}
	As an example of the structure of a Rete network compiled from a TBox, consider the following normalized TBox:
	\begin{align*}
		\T_{ex} = \{~ A_2 \sqcap A_4 \sqcap A_5 \sqcap \forall r_1. A_3 \sqcap \forall r_1. A_4 \sqcap \forall r_2. A_1 & \subsumed B_7, \\
		\forall r_2. A_3 \sqcap \forall r_2. A_4 & \subsumed B_8, \\
		\forall r_1. A_6 &\subsumed \forall r_1. B_9~ \}.
	\end{align*}
	The corresponding Rete network is displayed in Figure \ref{fig:rete} with the root node (Layer 1) and the three leaves being terminal nodes representing the left-hand sides of the three GCIs (Layer 5). The intermediate nodes are concept nodes representing (conjunctions of) named concepts (Layer 2), role nodes (Layer 3) or the inter-element node representing the conjunction of value restrictions for different roles from the first GCI in $\T_{ex}$ (Layer 4). 
\end{example}
\begin{figure} 
	\begin{center}
\fbox{\parbox[t]{0.8\textwidth}{
\vspace{\smallskipamount} \hfill
\begin{tikzpicture}
	\node[rootnode] (root) at (0,0) {root};
	\node[conceptnode] (a1) at (-3,-1.25) {$A_1$};
	\node[conceptnode] (a2) at (-1,-1.25) {$A_2$};
	\node[conceptnode] (a3) at (1,-1.25) {$A_3$};
	\node[conceptnode] (a6) at (3,-1.25) {$A_6$};
	
	\node[conceptnode] (a41) at (-1,-2.5) {$A_4$};
	\node[conceptnode] (a42) at (1,-2.5) {$A_4$};
	
	\node[conceptnode] (a5) at (-1,-3.75) {$A_5$};
	
	\node[rolenode] (r21) at (-3,-5) {$r_2$};
	\node[rolenode] (eps) at (-1,-5) {$\epsilon$};
	\node[rolenode] (r11) at (1,-5) {$r_1$};
	\node[rolenode] (r22) at (2,-5) {$r_2$};
	\node[rolenode] (r12) at (3,-5) {$r_1$};
	
	\node[internode] (inter) at (-1,-6) {$(\epsilon,r_1,r_2)$};
	
	\node[terminalnode] (b7) at (-1,-7) {$B_7$};
	\node[terminalnode] (b8) at (2,-7) {$B_8$};
	\node[terminalnode] (b9) at (3,-7) {$\forall r_1. B_9$};
	
	\draw[->] (root) -- (a1);
	\draw[->] (root) -- (a2);
	\draw[->] (root) -- (a3);
	\draw[->] (root) -- (a6);
	
	\draw[->] (a1) -- (r21);
	\draw[->] (a2) -- (a41);
	\draw[->] (a41) -- (a5);
	\draw[->] (a5) -- (eps);
	
	\draw[->] (a3) -- (a42);
	\draw[->] (a42) -- (r11);
	\draw[->] (a42) -- (r22);

	\draw[->] (a6) -- (r12);
	
	\draw[->] (r21) -- (inter);
	\draw[->] (eps) -- (inter);
	\draw[->] (r11) -- (inter);
	\draw[->] (inter) -- (b7);
	
	\draw[->] (r22) -- (b8);
	\draw[->] (r12) -- (b9);
\end{tikzpicture}
 \hfill \vspace{\medskipamount}}} 
\end{center}
	\caption{Rete network for the TBox $\T_{ex}$ from Example \ref{ex:rete}. Intermediate nodes are drawn in round shapes: concept nodes in light circles, role nodes in dark circles, and the inter-element node in an ellipsis. Terminal nodes are in displayed as boxes. }
	\label{fig:rete}
\end{figure}
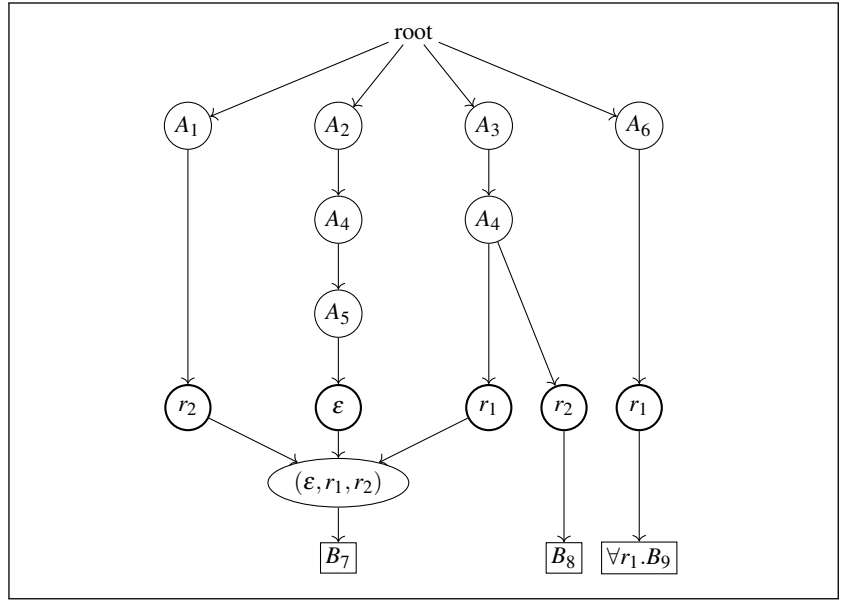

In the preprocessing phase, \flower compiles the normalized TBox \T into the corresponding \emph{Rete network}. In the main reasoning phase, \flower
saturates the initial partial functional interpretation $\Y_0$ by the Rete algorithm. 
To unleash the full potential of this Rete-based approach, we need to store the current model in a way that allows for fast access of its successor nodes, which is discussed in the next subsection.

\subsection{Numerical Representation of Partial Functional Interpretations}

\newcommand{\indx}[1]{\textsf{index}(#1)}

The operations \flower needs to perform repeatedly on the current partial functional interpretation $\Y_i$ for a given domain element are the following: 
\begin{enumerate}
	\item  quickly access its direct successors when a GCI is applied, 
	\item  quickly decide whether a smaller domain element with the same label set exists to test Condition \bRule{2} for direct blocking, and 
	\item  quickly decide whether a domain element is an ancestor of another element to test Condition \bRule{3} for indirect blocking.
\end{enumerate}

To obtain a space-efficient representation of the partial functional interpretation
that supports these operations with minimal overhead, we use an integer-based 
representation with the basis $\lvert\NR\rvert+1$. Specifically, we fix 
an enumeration of the role names in $\T$: $\NR=\{r_1,\ldots,r_n\}$. A 
word $\sigma=r_{i_1}r_{i_2}\ldots r_{i_m}\in\NRstar$ is then represented as 
$$\indx{\sigma}=\sum_{1\leq j\leq m} i_j(n+1)^{m-j}.$$ 
This representation reduces various operations on 
words that are relevant for the algorithm to fast arithmetic operations in the following ways:
\begin{enumerate}
	 \item the length of $\sigma$ is 
	$\left\lvert\sigma\rvert=\lfloor\log_{n+1}(\indx{\sigma})\right\rfloor$, 
	\item 
	the $r_i$-successor of $\sigma$ has index: $(n+1)\cdot\indx{\sigma}$,
	 \item the direct predecessor of $\sigma$ has index 
	$\left\lfloor\frac{\indx{\sigma}}{n+1}\right\rfloor$, and
	 \item checking whether $\rho$ is an ancestor of $\sigma$, i.e.\ whether
	$\sigma\in\pprefixset{\rho}$, can be done by checking whether 
		$$\indx{\sigma}=\left\lfloor\frac{\indx{\rho}}{(n+1)^{(\lvert\rho\rvert-\lvert\sigma\rvert)}}\right\rfloor.$$
\end{enumerate}
Note that this numerical encoding also directly provides an \emph{ordering on elements} $\prec$ as required: specifically, we define this ordering by 
$\sigma\prec\rho$ iff $\indx{\sigma}<\indx{\rho}$.

The labels of each domain element are stored in a tree map, which is a data structure that associates each index 
with a non-empty label to its label set. The inverse of this map is also 
stored, to quickly obtain which domain elements have a given label set. This operation is required to test the blocking Condition \bRule{2}.

\subsection{Implementation of  Blocking}\label{ssec:implementation-blocking}

After each expansion step it needs to be tested whether the blocking conditions \bRule{1} to \bRule{3} are fulfilled for the elements of the partial functional interpretation. Unfortunately, there can be intricate interactions between the blocking statuses of different elements.
Although GCIs are only applied on elements that are not blocked, the labels of a blocked element can change if a GCI is applied on some of its predecessors. As elements that are themselves blocked cannot block other elements, such a change in the label of a blocked element can lead to chain-reactions where the blocking status of a number of elements changes once information is propagated into a single element. 
An example of this effect is visualized in Figure~\ref{fig:Block}. On the left-hand side, $\sigma_3$ 
blocks $\sigma_4$, which makes the nodes $\sigma_5$ and $\sigma_6$ indirectly 
blocked. 
\begin{figure}
	\begin{centering}
\begin{tikzpicture}
 \node[conceptnode] (s1) at (0,0) {$\sigma_1$};
 \node[conceptnode] (s2) at (-0.5,-1) {$\sigma_2$};
 \node[conceptnode,very thick] (s3) at (0.5,-1) {$\sigma_3$};
 \draw[->] (s1) -- (s2);
 \draw[->] (s1) -- (s3);
 
 \node (dots1) at (2,0) {$\vdots$};
 \node[conceptnode,fill=black!30] (s4) at (2,-1) {$\sigma_4$};
 \node[conceptnode,fill=black!20] (s5) at (1.5, -2) {$\sigma_5$};
 \node[conceptnode,fill=black!20] (s6) at (2.5,-2) {$\sigma_6$};
 \draw[->] (s4) -- (s5);
 \draw[->] (s4) -- (s6);
 
 \path[->,dotted,very thick] (s3) edge [bend left] (s4);
 
 \node (dots2) at (4,0) {$\vdots$};
 \node[conceptnode] (s7) at (4,-1) {$\sigma_7$};
 \node[conceptnode] (s8) at (3.5, -2) {$\sigma_8$};
 \node[conceptnode] (s9) at (4.5,-2) {$\sigma_9$};
 \draw[->] (s7) -- (s8);
 \draw[->] (s7) -- (s9);
 
 \draw (-1,0.5) rectangle (5,-2.5);
 
 \draw[-{Triangle[width=18pt,length=8pt]}, line width=10pt](5.2,-1) -- (5.9, -1);

 \node[conceptnode,fill=black!30] (s1b) at (7,0) {$\sigma_1$};
 \node[conceptnode,fill=black!20] (s2b) at (6.5,-1) {$\sigma_2$};
 \node[conceptnode,fill=black!20] (s3b) at (7.5,-1) {$\sigma_3$};
 \draw[->] (s1b) -- (s2b);
 \draw[->] (s1b) -- (s3b);
 
 \node (dots1) at (9,0) {$\vdots$};
 \node[conceptnode] (s4b) at (9,-1) {$\sigma_4$};
 \node[conceptnode] (s5b) at (8.5, -2) {$\sigma_5$};
 \node[conceptnode,very thick] (s6b) at (9.5,-2) {$\sigma_6$};
 \draw[->] (s4b) -- (s5b);
 \draw[->] (s4b) -- (s6b);
 
 \node (dots2) at (11,0) {$\vdots$};
 \node[conceptnode] (s7b) at (11,-1) {$\sigma_7$};
 \node[conceptnode] (s8b) at (10.5, -2) {$\sigma_8$};
 \node[conceptnode,fill=black!30] (s9b) at (11.5,-2) {$\sigma_9$};
 \draw[->] (s7b) -- (s8b);
 \draw[->] (s7b) -- (s9b);
 
 \path[->,dotted,very  thick] (s6b) edge [bend right] (s9b);
 
 \draw (6,0.5) rectangle (12,-2.5);
\end{tikzpicture}

%
%
\end{centering}
	\caption{Example of blocking interactions in a partial functional interpretation. Gray elements are blocked, and the dotted arrow indicates an element directly blocking another.}
	\label{fig:Block}
\end{figure}
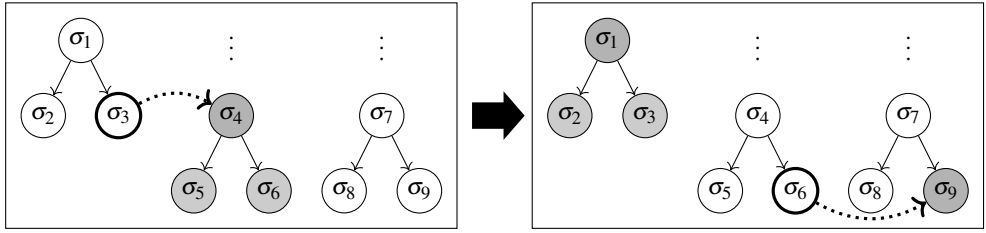
Thus, these nodes cannot block other nodes themselves. In our 
example, we assume $\sigma_6$ and $\sigma_9$ to have the same labels. 
$\sigma_9$ is not blocked by $\sigma_6$, since $\sigma_6$ is blocked. Blocking 
$\sigma_9$ 
would thus make the overall reasoning procedure incomplete. Now imagine some extension 
makes the node $\sigma_1$ blocked. The resulting situation is shown on the 
right-hand side. Since $\sigma_3$ becomes indirectly blocked, it cannot block 
$\sigma_4$ anymore. Consequently, also the 
	descendants
of $\sigma_4$ become 
unblocked, and now $\sigma_9$ becomes blocked by $\sigma_6$, even though there 
is no connection between these nodes and $\sigma_1$.

To determine directly blocked nodes, \flower uses two hash maps. One is mapping each node to 
its label set, and the other one is mapping each label set to a node. In addition, \flower stores 
for each node whether it is blocking another node, directly blocked, or 
indirectly blocked. If the label set of a node changes, \flower determines via the 
hash maps whether this change results in directly blocking or unblocking any nodes, and updates their 
blocking status accordingly. For every node whose blocking status changes, the indirect blocking status of their successors is recursively updated. If the 
indirect blocking status changes (as seen for $\sigma_6$ in the last example), \flower checks via the hash maps whether the 
blocking status of other nodes has to change as well, and invokes those changes. 
This process is continued recursively until all affected blocking statuses have 
been updated. 

\section{Evaluation of the \flower reasoner}\label{sec:eval}

The \flower reasoner is implemented in Java. It takes as input a general \FLbot TBox \T in OWL format \cite{DBLP:journals/ws/GrauHMPPS08} and normalizes the input TBox. If the ontology uses $\top$ or $\bot$, the transformation rules from Section~\ref{ssec:reductions} are applied. \flower realizes the following reasoning tasks. 
\begin{description}
	\item[\emph{Subsumption:}] Given two OWL classes $A$ and $B$, decide whether $A \subsumedT B$ holds.
	\item[\emph{Subsumer set:}] Given an OWL class $A$, compute all classes $B$ in $\T$ for which $A \subsumedT B$ holds.
	\item[\emph{Classification:}] Decide for all pairs of named OWL classes $A$ and $B$ occurring in $\T$ whether the subsumption $A \subsumedT B$ holds.
\end{description}
To decide subsumption \flower runs $\SUBSpar{A_0}{B_0}{\T}$, but stops as soon as the {subsumer candidate} $B_0$ occurs at the root of the tree. 
For computing the whole subsumer set of $A_0$, a single complete run of $\SUBSpar{A_0}{B_0}{\T}$ is sufficient, where the choice of $B_0$ is actually irrelevant. All subsumers of $A_0$ can be found at the root of the final tree.
Classification is done by running $\SUBSpar{A_0}{\tt{*}}{\T}$ for each named class $A_0$ in \T separately. (Again, the used {subsumer candidate} $B_0$ is irrelevant {for this kind of reasoning task and can be replaced by any concept other than $A_0$ or $\top$. This is indicated here by the wildcard $*$.}) The Rete network for \T is created only once and is reused for the remaining runs of $\SUBSpar{\tt{*}}{\tt{*}}{\T}$ {during classification}. Furthermore, {\flower} uses caching to reuse precomputed subsumer sets.

Our evaluation of \flower investigates two aspects. First, we wanted to see which optimizations implemented in \flower turned out to be effective and, second, we wanted to see how \flower's performance compares to other state-of-the-art DL reasoner. As most other DL reasoners that can handle (extensions of) \FLnull implement tableau-based methods, such a comparison would also tell {us} whether our new approach based on least functional models is competitive in terms of performance. 
We report on both kinds of evaluations in this section.
In order to be able to assess the performance of \flower, we needed to find suitable test ontologies first.

\subsection{Test Data}

\newcommand{\corpusA}{\textsc{Ore-Corpus}\xspace}
\newcommand{\corpusB}{\textsc{Mowl-Corpus}\xspace}
We generated two corpora for our evaluation. 
The first corpus \corpusA is based on the ontologies of the 
OWL EL classification track from the OWL Reasoner Evaluation 2015 (ORE 2015)~(see \cite{parsia2017owl}). The benchmarks of the ORE 2015 have the advantage that they have been balanced according to different criteria such as size, expressivity and complexity, and consist of many application ontologies. However, unfortunately no track is dedicated to ontologies in \FLbot. We thus generated \FLbot ontologies from ontologies written in \EL by \enquote{flipping} the quantifier, that is, by replacing $\exists$ by $\forall$. We furthermore dropped axioms involving role inclusions, nominals, or other operators that cannot be expressed in \FLbot. From the resulting corpus, we removed all ontologies with less than 500 concept names, resulting in a set of 209 \FLbot ontologies.\footnote{In the initial study on \flower's performance \cite{MTZ-RuleML-19} was an undetected bug which lead to more ontologies being discarded. There we used only 159 ontologies.}

While the ontologies used by the ORE do not contain a lot of \FLbot axioms, we found larger usage of them in the Manchester Ontology Corpus (MOWLCorp), which is a large ontology corpus containing 34,741 OWL
ontologies that were obtained by web-crawling~\cite{MOWLCORP}. The second corpus, \corpusB, is based on MOWLCorp. From each ontology in MOWLCorp, we removed axioms that could not be expressed in \FLbot. 
If the resulting ontology contained at least 500 concept names, it was included in our corpus. 
This resulted in a set of 382 \FLbot ontologies.
While the \corpusA contains more complex axioms and a more balanced set of ontologies, the \corpusB contains axioms that were obtained from application ontologies without modifications and thus preserves the original way of modeling. 

Figure~\ref{fig:corpora} shows the distribution of different parameters in the two corpora: number of concept names and number of axioms.
The largest ontology in \corpusA has 3,137,899 axioms, while the largest ontology in \corpusB has 279,682 axioms. 

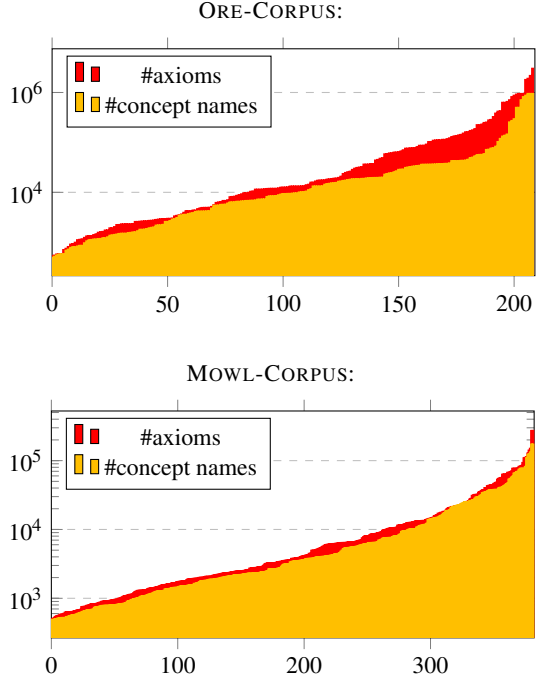
\begin{figure}
 \centering

\corpusA:

\medskip

\begin{tikzpicture}

  \begin{axis}[height=4.6cm, 
               width=8cm,
               ylabel={},
               xlabel={},
               ymode=log,
               ybar,
               bar width=0.04cm,
               ymajorgrids=true, grid style=dashed,
               xmin=0,
               xmax=209,
               legend pos=north west
              ]
    \addplot[draw=none, fill=red, bar shift=0] table 
    [y index = 0, x expr =\coordindex]
         {data/ore-axioms.txt}; 
    \addplot[draw=none, fill=orange!50!yellow, bar shift=0] table 
    [y index = 0, x expr=\coordindex]
        {data/ore-concepts.txt}; 
    \addlegendentry{\#axioms};
    \addlegendentry{\#concept names};
  \end{axis}
\end{tikzpicture}

\bigskip

\corpusB:

\medskip

\begin{tikzpicture}

  \begin{axis}[height=4.6cm, 
               width=8cm,
               ylabel={},
               xlabel={},
               ymode=log,
               ybar,
               bar width=0.04cm,
               ymajorgrids=true, grid style=dashed,
               xmin=0,
               xmax=382,
               legend pos=north west
              ]
    \addplot[draw=none, fill=red, bar shift=0] table 
    [y index = 0, x expr =\coordindex]
         {data/mowlcorp-axioms.txt}; 
    \addplot[draw=none, fill=orange!50!yellow, bar shift=0] table 
    [y index = 0, x expr=\coordindex]
        {data/mowlcorp-concepts.txt}; 
    \addlegendentry{\#axioms};
    \addlegendentry{\#concept names};
  \end{axis}
\end{tikzpicture}
 \caption{Numbers of axioms and concept names in the ontologies under consideration. 
 {The $y$-axis shows in logarithmic scale the number of axioms and concept names of the respective ontology, 
 for which we ordered the  values along the x-axis.}
 }
 \label{fig:corpora}
\end{figure}

In order to evaluate the subsumption task, we  generated 80 individual subsumption tests per ontology for the \corpusA, composed of 40 tests with \emph{positive} and 40 tests with \emph{negative} outcome. 
Since this resulted in a large number of reasoning experiments to be performed, this experiment was only performed on \corpusA, for which we assumed the most insights due to its more varied nature compared to \corpusB. 

Positive tests were generated by randomly selecting a concept name, and then randomly selecting a subsumer of it. Negative tests were generated by randomly selecting a concept name, and then randomly selecting another concept name that does not subsume the first.
For \flower, positive subsumption tests are easier, as the reasoner stops as soon as the subsumption relation has been proven. For tableau-based reasoning systems, the expected behavior is the other way around, as these reasoners try to create a counter-example to contradict the subsumption to be tested. Thus, evaluation results might not be as informative if one would just generate pairs for the subsumption test randomly without distinguishing between positive and negative tests as it was done in the earlier study  \cite{MTZ-RuleML-19}.

\subsection{Evaluation Setup}

In the initial study on \flower, we compared the performance for all three implemented reasoning tasks \emph{Subsumption}, \emph{Subsumer Set} and \emph{Classification} ~\cite{MTZ-RuleML-19}. 
Although the OWL API supports computing subsumer sets and is implemented by all considered reasoner systems, it would not yield an informative test, since all three tableaux-based systems classify the entire ontology before returning the subsumer set, as inspection of their source code revealed. This makes a comparison simply unfair and less insightful, which is why we restricted this study to the tasks \emph{Subsumption} and \emph{Classification}. 

For subsumption tests, we used the \corpusA and the concept pairs for positive and negative subsumptions. For each subsumption test, the timeout was set to 1 minute.
Both \corpusA and \corpusB were used to evaluate the  classification task. For each classification reasoning task, the timeout was set to 10 minutes. 

In addition to the running times measured for the two reasoning tasks, we also compared the computed results. While this comparison is easy for the \emph{Subsumption} tests, for \emph{Classification}, we computed a checksum for the classification result, and checked whether it was the same for every reasoner.

As a test system we have used an Intel Core i5-4590 CPU machine with 3.30GHz and 32 GB RAM, using Debian/GNU Linux 9 and OpenJDK 11.0.5.
Java was called with \textit{-Xmx8g} to set the maximum allocation pool (heap) size to 8 GB. 
We only measured  the running time of the actual reasoning task and not the time for loading the ontology.

\subsection{Evaluating \flower's Optimizations}

Although the current version of  \flower is certainly not highly optimized, it implements several optimizations.
We first evaluated the effect of the different optimizations within \flower:
\begin{description}
 \item[Multithreading] The main algorithm computes all subsumers for a given concept name. For the task of \emph{classification}, we partition the concept names into batches of size 48 plus one partition for the rest, and compute the subsumers for each partition in a different thread.
 \item[Ancestor blocking] Ancestor blocking corresponds to the blocking condition~\bRule{3}. While the method might not terminate without this condition, it is the interaction of this blocking condition with~\bRule{2} that makes the implementation of blocking more challenging (see Section~\ref{ssec:implementation-blocking}). To {assess} the impact of this blocking condition, we allowed to deactivate ancestor blocking in the implementation.
 \item[Role filtering] We do not generate $r$-successors for roles $r\in\NR$ that do not occur on the left-hand side of a GCI. 
 {As pointed out in Remark~\ref{rem:horn}, this optimization preserves soundness and completeness. Moreover, as shown in the proof for Theorem~\ref{the:horn-fl0}, reasoning in \HornFLnull becomes polynomial with this optimization, so that one would {expect} a big impact of this optimization.}
 \item[Global caching] When performing classification, we store previously computed subsumer sets. If a node with a concept name is added for which we already have a subsumer set, we add all the subsumers to that node and block it. The node only becomes unblocked when new concept names are added to its label by subsequent reasoning steps.
\end{description}
\newcommand{\flowerMT}{\flower-MT\xspace}
\newcommand{\flowerMTAB}{\flower-MT-AB\xspace}
\newcommand{\flowerMTABGC}{\flower-MT-AB-GC\xspace}
\newcommand{\flowerALL}{\flower-ALL\xspace}
We compared the following configurations of our reasoner:
\begin{itemize}
	\item \flower with no optimizations,
	\item \flowerMT (multithreading activated), 
	\item \flowerMTAB (multithreading and ancestor blocking), 
	\item \flowerMTABGC (multithreading, ancestor blocking and global caching), and 
	\item \flowerALL with all four optimizations activated.
\end{itemize}
Figure~\ref{fig:subsumption-flower} shows the results for the subsumption experiment, while Figure~\ref{fig:classification-flower} shows the results for the classification experiments. Here and in the figures that follow, we use logarithmic scaling on both axes, and we show for the runs that caused a timeout the maximal value (1 minute for subsumption, and 10 minutes for classification).

\begin{figure}
\begin{tabular}{c|r|r|r|r|r}
 & \flower & \flowerMT & \flowerMTAB & \flowerMTABGC & \flowerALL \\
Timeouts:  & 0.40\% & 0.39\% & 0.00\% & 0.00\% & 0.00\%
\end{tabular}

\begin{center}
 \includegraphics[width=.75\linewidth]{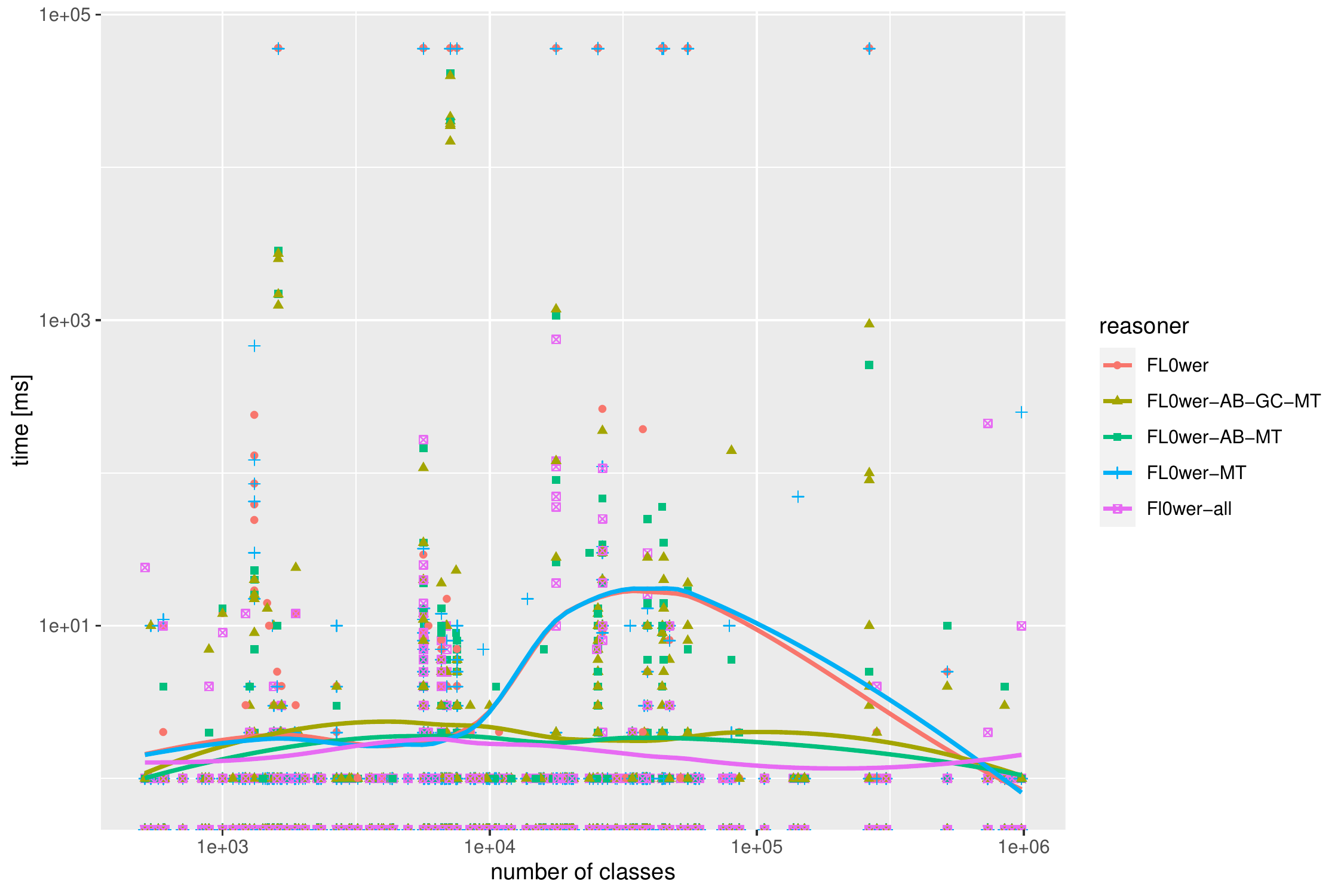}
\end{center}

\caption{Timeouts and running times for subsumption tests w.r.t.\ \corpusA with the different optimizations.}
\label{fig:subsumption-flower}
\end{figure}    

\begin{figure}

\corpusA:

\medskip

\begin{tabular}{c|r|r|r|r|r}
 & \flower & \flowerMT & \flowerMTAB & \flowerMTABGC & \flowerALL \\
Timeouts:  & 18.28\% & 17.91\% & 4.10\% & 2.99\% & 2.61\%
\end{tabular}

\begin{center}
 \includegraphics[width=.75\linewidth]{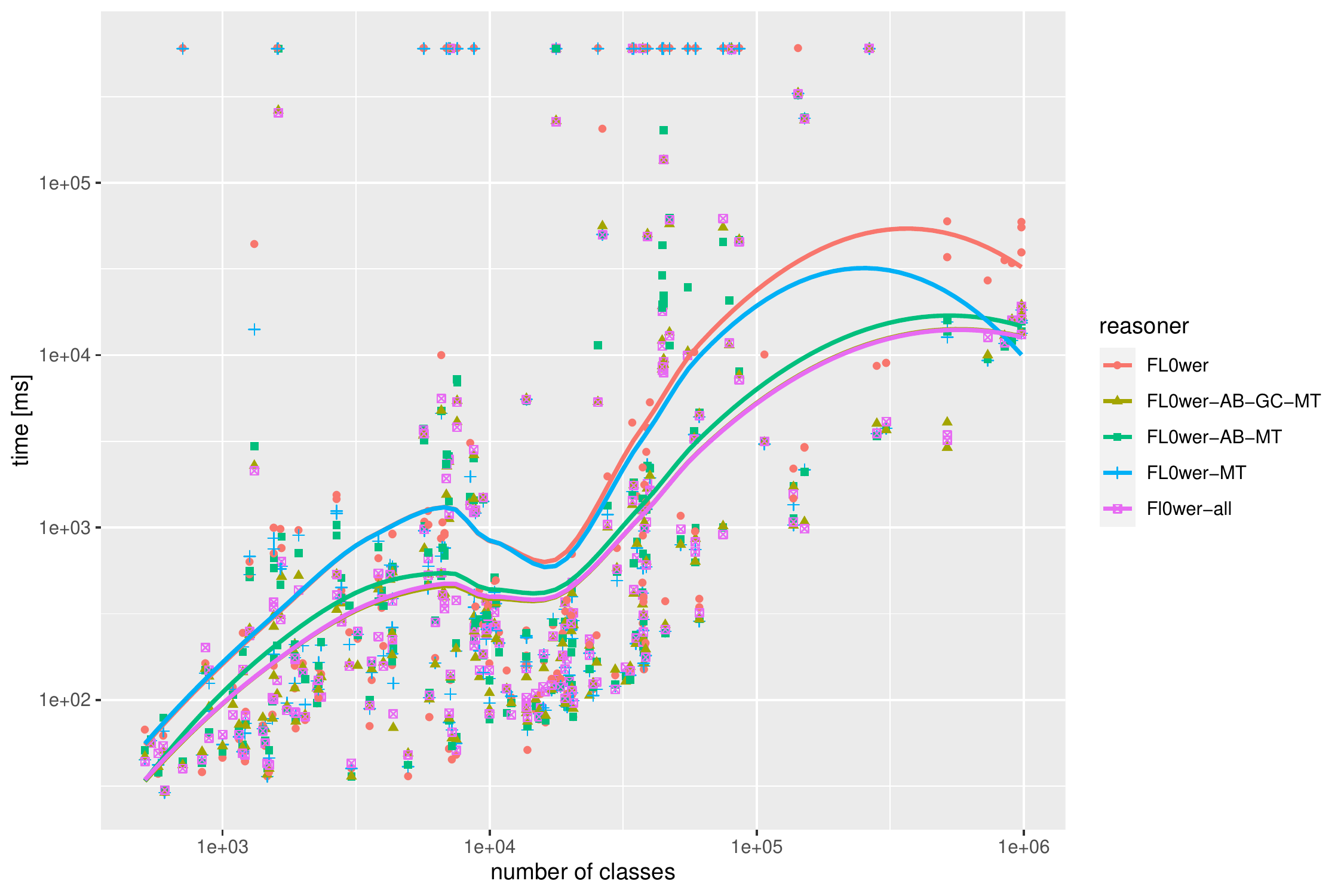}
\end{center}

\medskip

\corpusB:

\medskip

\begin{tabular}{c|r|r|r|r|r}
 & \flower & \flowerMT & \flowerMTAB & \flowerMTABGC & \flowerALL \\
Timeouts:  & 0.00\% & 0.00\% & 0.00\% & 0.00\% & 0.00\%
\end{tabular}

\begin{center}
 \includegraphics[width=.75\linewidth]{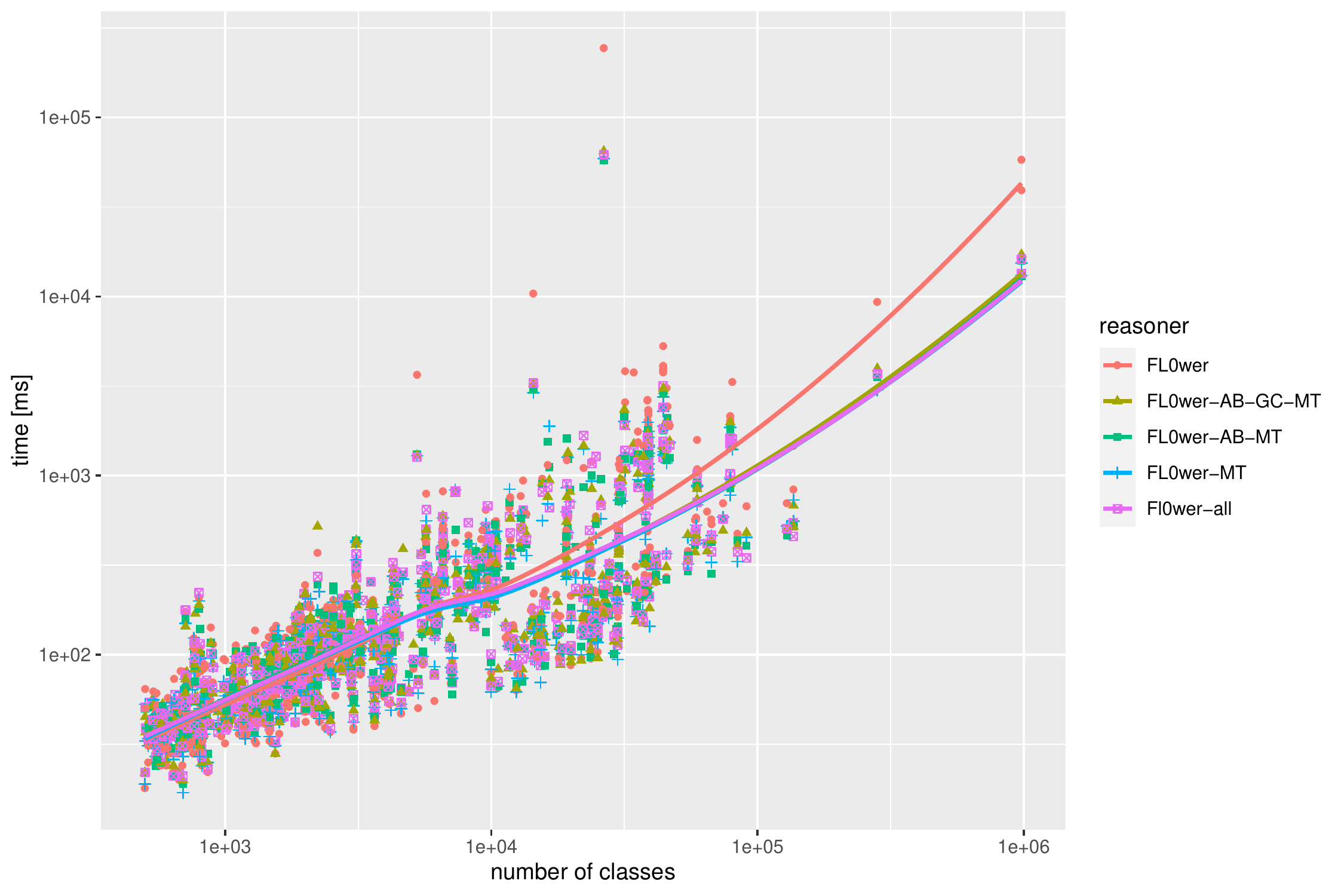}
\end{center}

\caption{Timeouts and running times for classification with the different optimizations. {Note that the curve for \flowerMTABGC is almost completely hidden under the curve for \flowerALL.}
}
\label{fig:classification-flower}
\end{figure}

{For the subsumption tests, the biggest impact was caused by ancestor-blocking, despite the additional obstacles in the implementation. On the other hand, considering that termination can only be guaranteed with ancestor-blocking activated, and that the additional blocking condition may lead to {fewer} nodes being generated, a positive effect was to be expected. In fact, for \corpusA, with ancestor blocking activated, the timeout rate dropped from 17.91\% to 4.10\%.
}
However, this positive effect was only notable for the ontologies in \corpusA, which can be explained by the simpler structure of the ontologies in \corpusB which in turn lead to simpler functional models. For a single subsumption task, the other optimizations merely seem to create an overhead and do not improve the performance in general. This is obvious for the caching procedure, which only brings a benefit if more than one subsumption task is performed. The largest impact here seems to be obtained by the role filtering. Interestingly, \flower's reasoning time seems hardly correlated with the number of classes in the ontology---only if this number becomes very large, optimizations seem to have even a negative impact.

For classification computed on the \corpusA, besides {ancestor blocking, global caching makes a noticeable impact, though it is not as large as one would expect for this task. In contrast, the impact of role filtering} is not as strong, though it decreases the number of timeouts from 2.99\% to 2.61\%. We also observe that for \corpusB, none of the optimizations apart from multithreading seem to be really indispensable. Again this can be explained by the simpler structure of the ontologies considered here.

\subsection{Comparison with other DL reasoners}

\newcommand{\hermit}{HermiT\xspace}
\newcommand{\openllet}{Openllet\xspace}
\newcommand{\jfact}{JFact\xspace}
We evaluated \flower to see how its reasoning times compare with those of other state-of-the-art DL reasoners. We used the configuration of our DL reasoner with all optimizations active: \flowerALL.
Since there is no other dedicated reasoner for \FLnull, we used reasoner systems that can handle expressive DLs of which \FLnull is a fragment. 
Here, we focused on reasoners which are implemented in Java just as \flower, and selected the following three state-of-the-art reasoning systems:
\begin{itemize}
	\item HermiT\footnote{\url{hermit-reasoner.com}}, version 1.3.8.510,
	\item Openllet\footnote{\url{github.com/Galigator/openllet}}, version 2.6.3, and
	\item JFact\footnote{\url{jfact.sourceforge.net}}, version 5.0.1.
\end{itemize} 
All three reasoners implement the OWL API \cite{DBLP:journals/semweb/HorridgeB11}, 
which allows us to measure and compare the time needed for the reasoning tasks alone---excluding the time
for loading the ontologies using the OWL API. 
Note that furthermore, all three reasoners implement tableaux-based algorithms, so that this comparative evaluation also serves as a comparison of the different approaches: least functional model generation vs.\ tableaux-based approach. 
Regarding the actual reasoner implementations, we note that these are all complex and mature reasoning \emph{systems} that come with more sophisticated optimizations than \flower, { which makes it even more surprising that \flower performs quite well in comparison.} The timeouts and reasoning times of \flower compared with those of the above three reasoners, are shown in Figure~\ref{fig:subsumption-all} for the subsumption experiment, and in Figure~\ref{fig:classification-all} for the two classification experiments. 
\begin{figure}
\begin{tabular}{c|r|r|r|r}
 & \hermit & \jfact & \openllet & \flowerALL \\
Timeouts:  & 0.001\% & 5.57\% & 1.53\% & 0.00\%
\end{tabular}

\begin{center}
 \includegraphics[width=.75\linewidth]{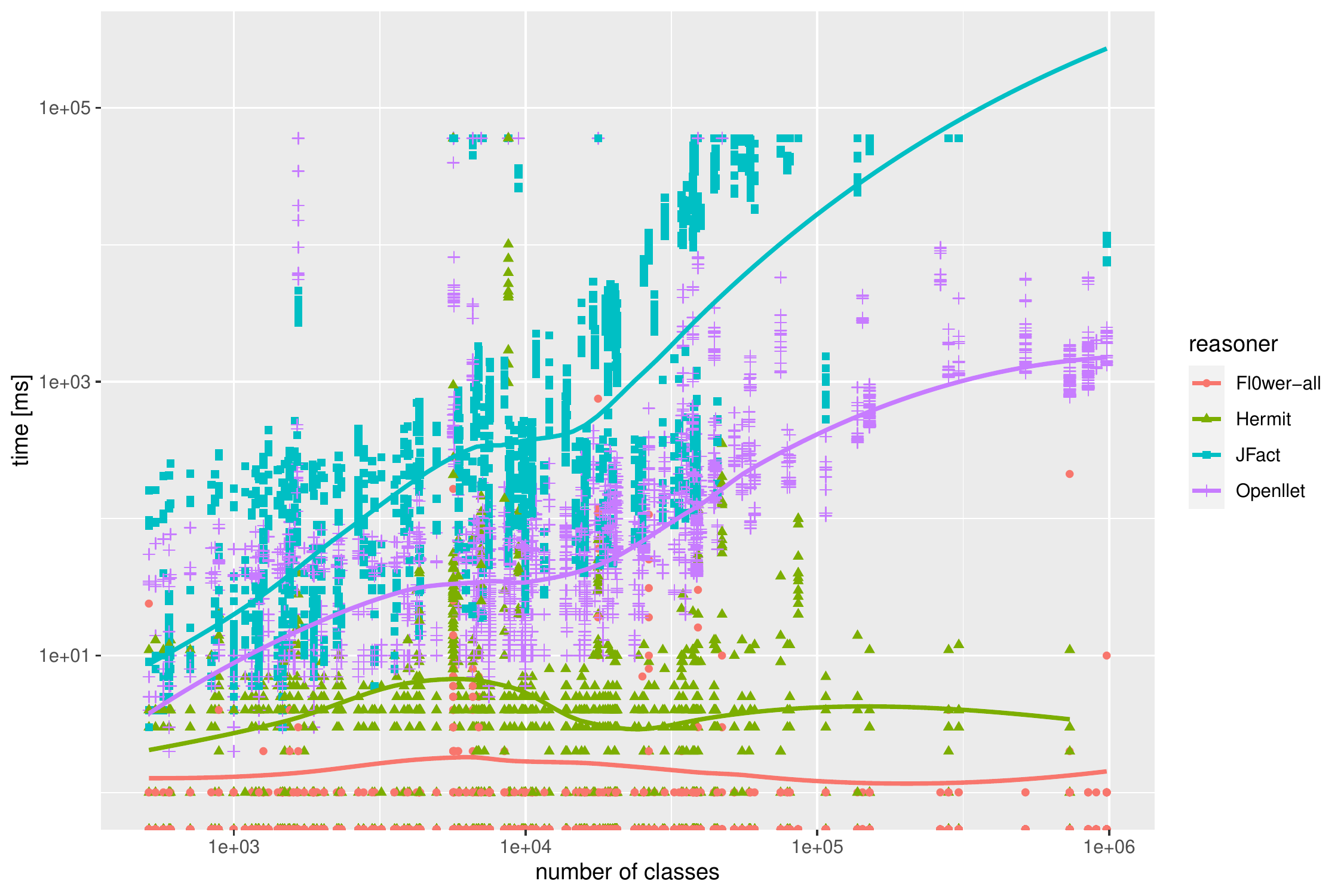}
\end{center}

\caption{Timeouts and running times for subsumption tests w.r.t.\ \corpusA for the different reasoners.}
\label{fig:subsumption-all}
\end{figure}    

Interestingly, for the subsumption tests, 
the performance of both \flower and \hermit seems hardly affected by the number of concept names in the ontology. This number has a much bigger impact for \jfact and \openllet which need orders of magnitude more running time than \flower and \hermit to decide subsumption for the test cases. 
However, while some subsumption tasks still lead to timeouts for \hermit in 0.0047\% of cases, no timeouts were observed by  \flower.  
Generally, \flower performs substantially better for this task (on this test set) than the other DL reasoners.

\begin{figure}
\corpusA:

\medskip

\begin{tabular}{c|r|r|r|r}
 & \hermit & \jfact & \openllet & \flowerALL \\
Timeouts:  & 9.50\% & 21.01\% & 4.78\% & 2.61\%
\end{tabular}

\begin{center}
 \includegraphics[width=.75\linewidth]{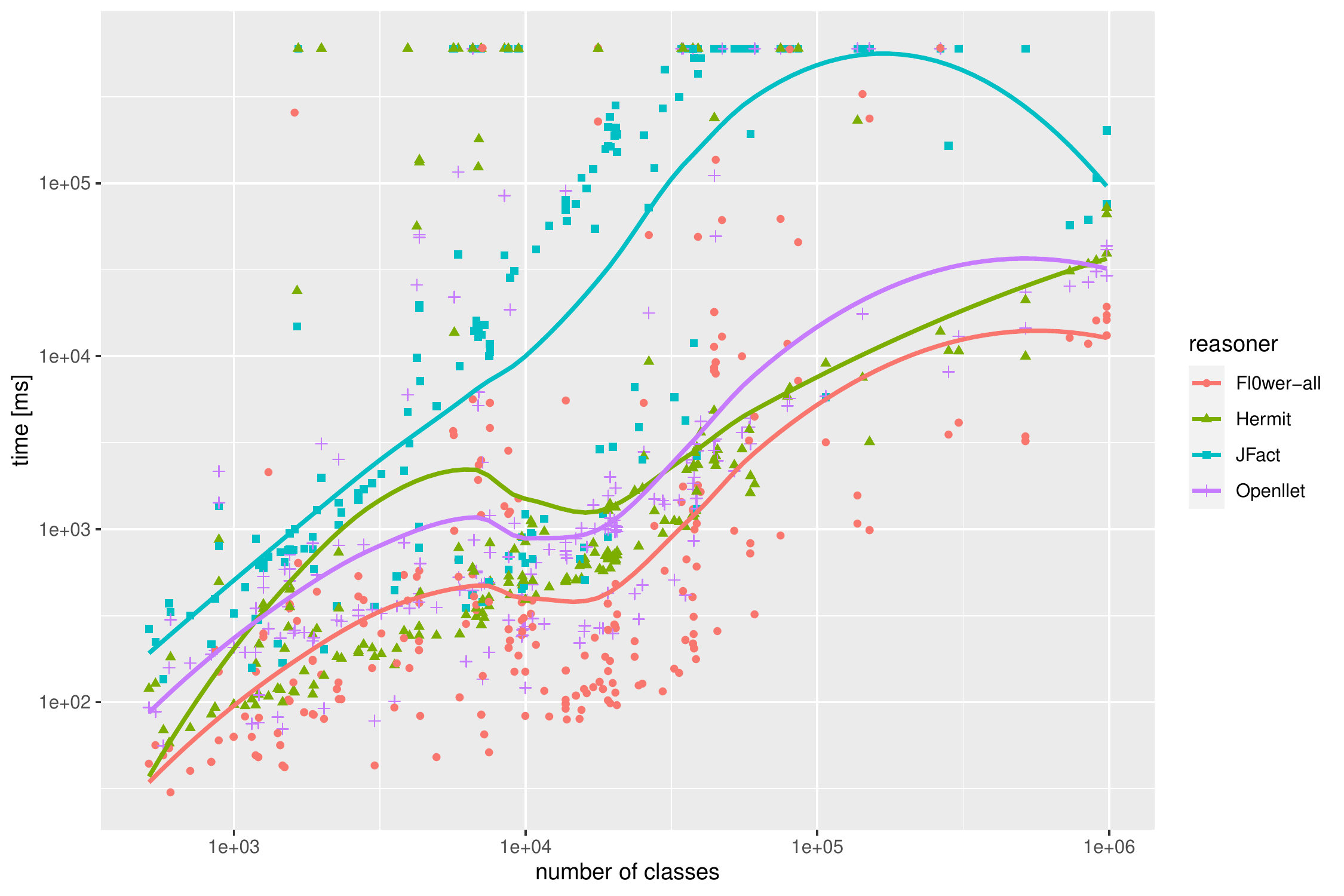}
\end{center}

\medskip

\corpusB:

\medskip

\begin{tabular}{c|r|r|r|r}
 & \hermit & \jfact & \openllet & \flowerALL \\
Timeouts:  & 0.00\% & 2.19\% & 0.00\% & 0.00\%
\end{tabular}

\begin{center}
 \includegraphics[width=.75\linewidth]{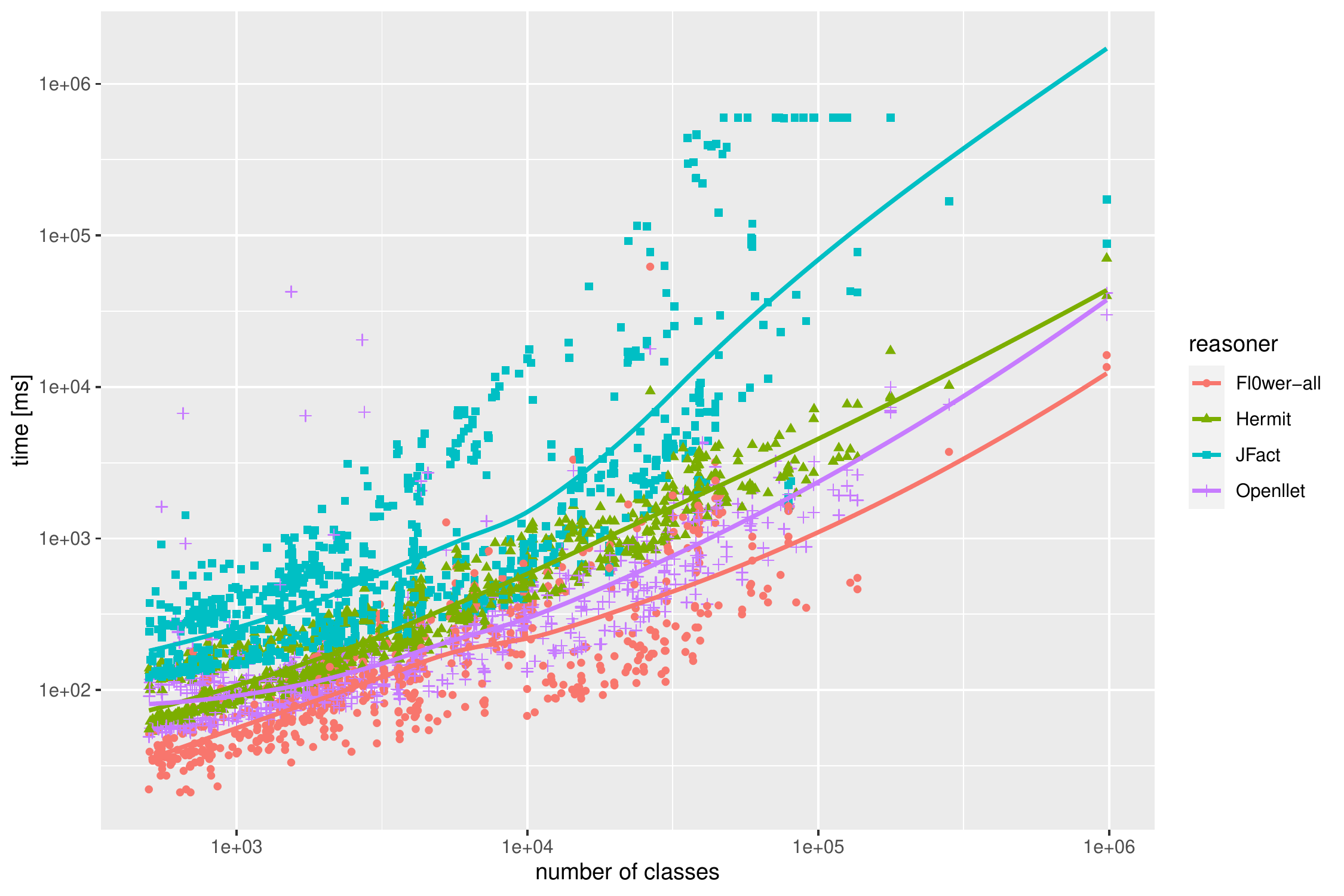}
\end{center}

\caption{Timeouts and running times for classification with the different reasoners.}
\label{fig:classification-all}
\end{figure}

For classification on \corpusA, our measurements indicate that \flower performs the best among the four systems. \jfact's running time is roughly an order of magnitude higher than the one of \flower.
\hermit has twice as many timeouts as \openllet, but the picture on running times is more mixed, where \hermit often performed better than \openllet. \flower again almost halved the number of timeouts compared to \openllet, but here, the performance looks consistently better than for all other reasoners. 
Interestingly, the (interpolated) performance curves of \flower, \hermit, and \openllet show very similar characteristics, as they develop almost synchronously. This may suggest that the same kind of ontology is difficult for all three systems and for both reasoning approaches. 
For \corpusB, the general picture is in principle similar. We can see a clear ranking between the reasoners, with \flower performing generally the best. For this corpus there were only timeouts for \jfact.
Again, the (interpolated) performance curves of \hermit and \openllet are similar to the one of \flower---albeit less strongly as in the case of the \corpusA. However, this may  
support the earlier finding that the same kind of ontology could be difficult for both reasoning approaches.

\smallskip
\noindent
To sum up, the running time  of \flower for testing subsumption and for computing classification is on average substantially better than the one of \jfact, \openllet, and even of \hermit. This is a remarkable result of a comparison between a newcomer system that implements only a few optimizations and well-established systems that have been developed for years. Our comparative evaluation suggests that the same kind of ontology may be difficult (or, alternatively, be easy) for reasoners based on the computation of least functional models as well as for tableaux-based reasoners.

\section{Conclusions}\label{sec:concl}

The main contribution of this paper is a novel algorithm for deciding subsumption in the DL \FLnull \wrt general TBoxes,
and a practical demonstration that this algorithm is easy to implement and behaves surprisingly well on large ontologies. Our reasoner \flower outperforms state-of-the art DL reasoners for testing subsumption and for classifying general TBoxes.

One may ask, however, why a dedicated reasoner for \FLnull is needed, given the facts that the worst-case complexity of reasoning
in \FLnull is as high as for the considerably more expressive DL \ALC and that there are very few pure \FLnull ontologies
available. We argue that such a dedicated reasoner may turn out to be very useful.
First, the latter fact could be due to a {chicken} and egg problem: as long as no dedicated reasoner for \FLnull is available,
there is no incentive to restrict the expressiveness to \FLnull when creating an ontology. When extracting our test ontologies,
we observed that quite a number of application ontologies have large \FLnull fragments.
Second, regarding the former fact, it is well-known in the DL community that worst-case complexity results are not always
a good indication for  how hard reasoning turns out to be in practice. 
Third, some DL reasoners such as 
Konclude\footnote{\url{konclude.com}} and 
MORe \cite{DBLP:conf/semweb/RomeroGH12} 
make use of specialized algorithms for certain language fragments as part of their overall reasoning approach, with
impressive improvements of the performance. Our efficient subsumption algorithm for \FLnull may turn out to be useful
in this context.
Finally, quite a number of non-standard reasoning tasks 
in \FLnull \wrt general TBoxes 
have recently been investigated \cite{DBLP:conf/jelia/BaaderMO16,DBLP:conf/lpar/BaaderGM18,DBLP:conf/gcai/BaaderGP18,DBLP:conf/www/BaaderMP18}.
The algorithms developed for solving these tasks usually depend on sub-procedures that perform subsumption tests or that use the least functional model directly.
Our reasoner \flower thus provides us with an efficient base for implementing such non-standard inferences.

\FloatBarrier



\bibliographystyle{tlplike}
\bibliography{bibliography,medium-string,dl}

\label{lastpage}
\end{document}